\documentclass{article}
\usepackage[accepted]{icml2023}
\usepackage{microtype}
\usepackage{graphicx}
\usepackage{caption}
\usepackage{subcaption}
\usepackage{booktabs} 
\usepackage{amsmath}
\usepackage{amsthm}
\usepackage{soul}
\usepackage[inline]{enumitem}

\graphicspath{ {./images/} }
\theoremstyle{plain}
\newtheorem{theorem}{Theorem}[section]
\newtheorem{proposition}[theorem]{Proposition}
\newtheorem{lemma}[theorem]{Lemma}
\newtheorem{corollary}[theorem]{Corollary}
\theoremstyle{definition}
\newtheorem{definition}[theorem]{Definition}
\newtheorem{assumption}[theorem]{Assumption}
\theoremstyle{remark}

\usepackage[utf8]{inputenc}
\usepackage{xcolor}
\usepackage{amsfonts}
\usepackage{hyperref}
\usepackage[capitalize,noabbrev]{cleveref}

\newcommand{\OLD}[1]{}

\newcommand{\histset}{\Omega}
\newcommand{\actionset}{\mathcal{A}}

\newcommand{\informationgenerator}{\sigma_t}
\newcommand{\infstatespace}{\mathcal{Z}}
\newcommand{\infstate}{Z}
\DeclareMathOperator*{\argmax}{argmax}

\newcommand{\rewardfunc}{\mathcal{R}}

\begin{document}

\twocolumn[
\icmltitle{ContraBAR: Contrastive Bayes-Adaptive Deep RL}
    \begin{icmlauthorlist}
\icmlauthor{Era Choshen}{technion}
\icmlauthor{Aviv Tamar}{technion}
\end{icmlauthorlist}

\icmlaffiliation{technion}{Technion, Haifa, Israel}
\icmlcorrespondingauthor{Era Choshen}{erachoshen@campus.technion.ac.il}
\vskip 0.3in
]
\printAffiliationsAndNotice{}

\begin{abstract}
In meta reinforcement learning (meta RL), an agent seeks a Bayes-optimal policy -- the optimal policy when facing an unknown task that is sampled from some known task distribution. Previous approaches tackled this problem by inferring a \textit{belief} over task parameters, using variational inference methods. Motivated by recent successes of contrastive learning approaches in RL, such as contrastive predictive coding (CPC), we investigate whether contrastive methods can be used for learning Bayes-optimal behavior. We begin by proving that representations learned by CPC are indeed sufficient for Bayes optimality. 
Based on this observation, we propose a simple meta RL algorithm that uses CPC in lieu of variational belief inference. 
Our method, \textit{ContraBAR}, achieves comparable performance to state-of-the-art in domains with state-based observation and circumvents the computational toll of future observation reconstruction, enabling learning in domains with image-based observations. It can also be combined with image augmentations for domain randomization and used seamlessly in both online and offline meta RL settings.
\end{abstract}

\section{Introduction}

In meta reinforcement learning (meta RL), an agent learns from a set of training tasks how to quickly solve a new task, sampled from a similar distribution as the training set~\cite{finn2017model,duan2016rl}. A formal setting for meta RL is based on the Bayesian RL formulation, where a task corresponds to a particular Markov decision process (MDP), and there exists some \textit{prior} distribution over MDPs~\cite{humplik2019meta,zintgraf2020varibad,rakelly2019efficient}. Under this setting, the \textit{optimal} meta RL policy is well defined, and is often referred to as a \textit{Bayes-optimal} policy~\cite{ghavamzadeh2016bayesian}. 

In contrast to the single MDP setting, where an optimal policy can be Markovian -- taking as input the current state and outputting the next action, the Bayes-optimal policy must take as input \textit{the whole history} of past states, actions, and rewards, or some sufficient statistic of it~\cite{bertsekas1995dynamic}. A popular sufficient statistic is the \textit{belief} -- the posterior probability of the MDP parameters given the observed history. For small MDPs, the belief may be inferred by directly applying Bayes rule, and approximate dynamic programming can be used to calculate an approximately Bayes-optimal policy~\cite{ghavamzadeh2016bayesian}. However, this approach quickly becomes intractable for large or continuous MDPs.
Recently, several studies proposed to scale up belief inference using deep learning, where the key idea is to leverage a variational autoencoder (VAE, \citealt{kingma2013auto}) formulation of the problem, in which the posterior is approximated using a recurrent neural network~\cite{zintgraf2020varibad,humplik2019meta}. While this approach has demonstrated impressive results on continuous control benchmarks~\cite{zintgraf2020varibad,humplik2019meta,dorfman2021offline}, it also has some limitations. Training a VAE is based on a \textit{reconstruction} loss, in this case, predicting the future observations given the current history, which can be difficult to optimize for visually rich observations such as images.
Furthermore, variational algorithms such as VariBAD~\citep{zintgraf2020varibad} reconstruct entire trajectories, restricting application to image-based domains due to memory limitations.


As an alternative to VAEs, contrastive learning has shown remarkable success in learning representations for various domains, including image recognition and speech processing~\cite{chen2020simple,fu2021scala,han2021supervised}. Rather than using a reconstruction loss, these approaches learn features that discriminate between similar observations and dissimilar ones, using a contrastive loss such as the InfoNCE in contrastive predictive coding (CPC, \citealt{oord2018representation}). Indeed, several recent studies showed that contrastive learning can learn useful representations for image based RL~\cite{laskin_srinivas2020curl,laskin2022cic}, outperforming representations learned using VAEs. Furthermore, \citet{guo2018neural} showed empirically that in partially observed MDPs, representations learned using CPC \citep{oord2018representation} are correlated with the belief. In this work, we further investigate contrastive learning for meta RL, henceforth termed \textit{CL meta RL}, and aim to establish it as a principled and advantageous alternative to the variational approach. 

Our first contribution is a proof that, given certain assumptions on data collection and the optimization process of CPC, representations learned using a variant of CPC are indeed a sufficient statistic for control, and therefore suffice  as input for a Bayes-optimal policy. Our second contribution is a bound on the suboptimality of a policy that uses an \textit{approximate} sufficient statistic, learned by CPC, in an iterative policy improvement scheme where policies between iterations are constrained to be similar. This result relaxes the assumptions on the optimization and data collection in the first proof.  Building on this result, we propose a simple meta RL algorithm that uses a CPC based representation to learn a sufficient statistic.
Our third contribution is an empirical evaluation of our method that exposes several advantages of the contrastive learning approach. In particular, we show that:
\begin{enumerate*}[label=(\arabic*)]
    \item For state-based observations, CL meta RL is on par with the state-of-the-art VariBAD~\citep{zintgraf2020varibad}
    \item For image-based observations, CL meta RL significantly outperforms the variational approach, and is competitive with RNN based methods~\cite{duan2016rl}
    \item In contrast to the variational approach, CL Meta RL is compatible with image augmentations and domain randomization. 
    \item Our method works well in the online and offline meta RL setting.
\end{enumerate*}
Overall, our results establish CL meta RL as a versatile and competitive approach to meta RL.

\section{Background and Problem Formulation}\label{section:background}
In this section we present our problem formulation and relevant background material.

\subsection{Meta RL and POMDPs}
We define a Markov Decision Process (MDP) \citep{bertsekas1995dynamic} as a tuple $\mathcal{M}=(\mathcal{S},\mathcal{A},\mathcal{P}, \rewardfunc)$, where $\mathcal{S}$ is the state space, $\mathcal{A}$ is the action space, $\mathcal{P}$ is the transition kernel and $\rewardfunc$ is the reward function. In meta RL, we assume a distribution over tasks, where each task is an MDP $\mathcal{M}_i=(\mathcal{S},\mathcal{A},\mathcal{P}_i,\rewardfunc_i)$, where the state and action spaces are shared across tasks, and $\mathcal{P}_i,\rewardfunc_i$ are task specific and drawn from a task distribution, which we denote $\mathcal{D}(\mathcal{P},\rewardfunc)$. At a given time $t$, we denote by {$(s_0, a_0,r_0, s_1,a_1,r_1,\dots,s_t)=h_t\in \histset_t$ the current history, where $\histset_t$ is the space of all state-action-reward histories until time $t$}. Our aim in meta RL is to find a policy $\pi = \left\{ \pi_0, \pi_1,\dots\right\}$, where $\pi_t:\histset_t \to \mathcal{A}$, which maximizes the following objective:
\[\mathbb{E}_{\pi}\left[\sum_{t=0}^{\infty} \gamma^t r_t\right],\]
where the expectation $\mathbb{E}_\pi$ is taken over the transitions $s_{t+1} \sim \mathcal{P}(\cdot|s_t,a_t)$, the reward $r_t = \rewardfunc(s_t, a_t)$, the actions $a_t \sim \pi(\cdot|h_t)$  and the uncertainty over the MDP parameters $\mathcal{P},\rewardfunc \sim \mathcal{D}(\mathcal{P},\rewardfunc)$. We assume a bounded  reward $r_t \in [-R_{max}, R_{max}], R_{max} > 0$ with probability one.

Meta RL is a special case of the more general Partially Observed Markov-Decision Process (POMDP), which is an extension of MDPs to partially observed states.
In the POMDP for meta RL, the unobserved variables are $\mathcal{P}, \rewardfunc$, and they do not change over time. We define $\histset_t$ for POMDPs as above, except that states are replaced by observations according to the distribution $o_{t+1} \sim U(o_{t+1}|s_{t+1},a_t)$. As shown in \citet{bertsekas1995dynamic}, the optimal policy for a POMDP can be calculated using backwards dynamic programming for every possible $h_t \in \histset_t$. However, as explained in \citet{bertsekas1995dynamic} this method is computationally intractable in most cases as $\histset_t$ grows exponentially with $t$.

\subsection{Information States and BAMDPs}

Instead of the intractable space of histories, \textit{sufficient statistics} can succinctly summarize all the necessary information for optimal control. One popular sufficient statistic is the posterior state distribution or \textit{belief} $P(s_t|h_t)$. Conditions for a function to be a sufficient statistic, also termed \textit{information state}, were presented by \citet{subramanian2020approximate} and are reiterated here for completeness:

\begin{definition}[Information State Generator] 
Let $\left\{\infstatespace_t\right\}_{t=1}^{T}$ be a pre-specified collection of Banach spaces. A collection $\left\{\informationgenerator: \histset_t \rightarrow \infstatespace_t \right \}_{t=1}^{T}$ of history compression functions is called an information generator if the process $\left \{ \infstate_t \right \}_{t=1}^T$ satisfies the following properties,  where $h_t \in \histset_t$, and $\informationgenerator(h_t)=\infstate_t \in \mathcal{Z}_t$:

\textbf{P1} For any time $t$ and for any $h_t \in \histset_t, a_t \in \actionset$ we have:
\[\mathbb{E}\left[r_t | h_t, a_t\right] = \mathbb{E}\left[r_t |\infstate_t=\informationgenerator(h_t), a_t \right].\]
\textbf{P2} For any time $t$, and for any $h_t \in \histset_t, a_t \in \actionset$, and any Borel subset $B$ of $Z_{t+1}$ we have:
\[P\left( B \in \infstate_{t+1} | h_t,a_t \right) = P \left( B \in \infstate_{t+1} | \infstate_t=\informationgenerator(h_t),a_t \right).\]
\end{definition}

Intuitively, information states compress the history without losing predictive power about the next reward, or the next information state.

To solve a POMDP, one can define a Bayes-Adaptive MDP (BAMDP)-- an MDP over the augmented state space of $\mathcal{S}\times \mathcal{B}$, where $\mathcal{B}=\{\mathcal{Z}_t \}_{t=1}^{T}$ is the space of the information state. This idea was introduced by \citet{duff2002optimal} for the belief. Here, we use the term BAMDP more generally, referring to any information state. The optimal policies for BAMDPs are termed Bayes-optimal and optimally trade-off between exploration and exploitation, which is essential for maximizing online return during learning. Unfortunately, in most cases computing the Bayes-optimal policy is intractable  because the augmented space is continuous and high-dimensional. \citet{zintgraf2020varibad} proposed to approximate the Bayes-optimal policy by using deep neural networks to learn an information state (belief), and conditioning an RL agent on the learned augmented space; here we follow this approach.


\section{Related Work}
Our focus in this work is learning a Bayes-optimal policy for meta RL. We recapitulate the current approaches to meta RL with a focus on approaches that potentially yield Bayes-optimal policies.

The methods in \cite{finn2017model, grant2018recasting, nichol2018gotta, rothfuss2018promp, clavera2018model} learn neural network policies that can quickly be fine-tuned to new tasks at test time via gradient updates. These methods do not optimize for Bayes-optimal behavior, and typically exhibit significantly suboptimal test-time adaptation.

A different approach is to learn an agent that directly infers the task at test time, and conditions the policy based on the inferred task. Typically, past interactions of the agent with the environment are aggregated to a latent representation of the task.
\citet{rakelly2019efficient, fakoor2019meta} follow a posterior-sampling approach, which is not Bayes-optimal~\cite{zintgraf2020varibad}; in this work we focus on methods that can achieve Bayes-optimality.
\citet{duan2016rl, wang2016learning} propose memory-based approaches, which \citet{ortega2019meta} proves to approximate Bayes-optimal agents. \citet{zintgraf2020varibad, zintgraf2021exploration, dorfman2021offline} also approximate Bayes-optimal agents with a history-based representation, using a variational approach. 
\citet{humplik2019meta} learn an approximately Bayes-optimal agent, where privileged information -- a task descriptor -- is used to learn a sufficient statistic.
We explore an alternative approach that lies at the intersection of meta RL and contrastive learning. Different from memory-based methods such as RL$^2$~\citep{duan2016rl}, and similarly to VariBAD, we learn a history based embedding \textit{separately} from the policy.
However, unlike variational methods, we learn the task representation using contrastive learning.

Contrastive learning has been used to learn representations for input to a meta RL policy. FOCAL~\citep{li2020focal} uses distance metric learning to learn a deterministic encoder of transition tuples to perform offline RL. They operate under the relatively restrictive assumption that each transition tuple $(s, a, s', r)$ is uniquely identified by a task. The authors followed up with FOCAL++,  in which batches of transition tuples (not necessarily from the same trajectory) are encoded to a representation that is optimized with MoCo \citep{he2020momentum}, a variant of CPC, alongside an intra-task attention mechanism meant to robustify task inference \citep{li2021provably}. The MBML method in \citep{li2020multitask} proposes an offline meta RL method that uses the triplet loss to learn embeddings of batches of transition tuples from the same task, with the same  probabilistic and permutation-invariant architecture of \citet{rakelly2019efficient}. \citet{wang2021improving} propose embedding  windows of transition tuples as probabilistic latent variables, where the windows are cropped from different trajectories. The embeddings are learned  with MoCO \citep{he2020momentum} by contrasting them in probabilistic metric space, where positive pairs are transition windows that come from the same batch. The algorithm is presented as a general method to learn representations for context-based meta RL algorithms, but in practice all results are shown with PEARL \citep{rakelly2019efficient}. In a similar line of work, \citet{fu2021towards} encode batches of transitions as a product of Gaussian factors and contrast the embeddings with MoCO \citep{he2020momentum}, with positive pairs being embedded transition batches from the same \textit{task}, as opposed to the same trajectory as in \citet{wang2021improving}. As in \citet{wang2021improving}, results are shown with a posterior sampling meta RL algorithm. While we also investigate contrastive learning for meta RL, we make an important distinction: all of the works above embed \textit{transition tuples and not histories}, and therefore cannot represent information states, and cannot obtain Bayes-optimal behavior. In contrast, in our work, we draw inspiration from \citet{guo2018neural}, who used a glass-box approach to empirically show that contrastive learning \textit{can} be used to learn the belief in a POMDP. We cast this idea in the Bayesian-RL formalism, and show both theoretically and empirically, that contrastive learning can be used to learn Bayes-optimal meta RL policies.

\section{Method}

\begin{figure*}[!h]
   \centering
  \vskip 0.2in
  \begin{center}
  \centerline{\includegraphics[width=\textwidth]{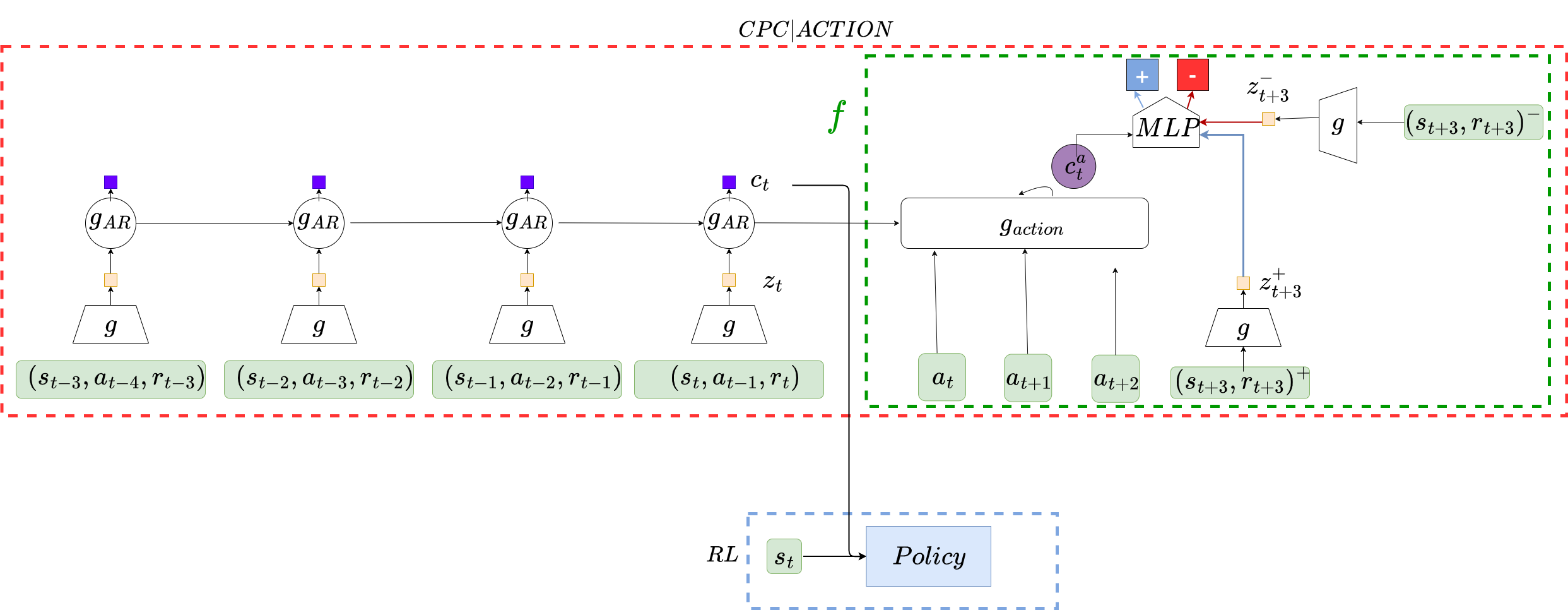}}
  \vspace{-1em}
  \caption{ContraBAR architecture. Red box: similarly to CPC, $g_{AR}$ is autoregressively applied to past observation embeddings ($z_t,z_{t-1},\dots$) to yield the current representation $c_t$. Green box: unlike CPC, the MLP $f$ receives as input, in addition to a positive/negative future embedding, the output of an action GRU that process $c_t$ and the future action sequence. Blue box: the representation $c_t$ is the information state, which is concatenated with the state $s_t$ and input to the RL policy.}\label{fig:contrabar_arch}
  \end{center}
\vskip -0.2in

\end{figure*}

In this section we show how to use contrastive learning to learn an information state representation of the history, and use it as input to an RL agent. We give a brief description of CPC~\citep{oord2018representation} followed by our meta RL algorithm. We then prove that our method does indeed learn an information state.
\subsection{Contrastive Predictive Coding}\label{subsection:cpc}

 CPC \citep{oord2018representation} is a contrastive learning method that uses noise contrastive estimation 
\citep{gutmann2010noise} to discriminate between positive future observations $o_{t+k}^{+}$, where $t$ is the current time step, and negative observations $o_{t+k}^{-}$. First, an encoder $g$ generates an embedding for each observation in a sequence of observations from a trajectory $\tau$ until time $t$, $\left \{z_i=g(o_i)\right\}_{i=1}^{t}$. Second, an autoregressive model $g_{AR}$ summarizes $z_{\leq t}$, the past $t$ observations in latent space, and outputs a latent $c_t$. The model is trained to discriminate between future observations  $o_{t+k}^{+}$ and $K$ negative observations $\left \{o_{t+k}^{-,i}\right\}_{i=1}^{K}$ given $c_t$. Given a set $X=\{o_{t+k}^{+},o_{t+k}^{-,1},\dots,o_{t+k}^{-,K}\}$ containing one positive future observation sampled according to $P(o_{t+k}^{+}|c_t)$ and $  K$ negative observations sampled from a proposal distribution $P(o_{t+k}^{-})$, the InfoNCE loss is:

\[\mathcal{L}_{InfoNCE}=\]
\begin{small}
\[-\mathbb{E}_{X}
\left[\log\frac{\exp \left(f\left(c_t,o_{t+k}^{+}\right)\right)}{\exp \left(f\left(c_t,o_{t+k}^{+}\right)\right) + \sum_{i=1}^{K} \exp \left(f\left(c_t,o_{t+k}^{-,i}\right)\right)}\right],\] 
\end{small}
where $f$ is a learnable function that outputs a similarity score. The model components $f, g, g_{AR}$ are learned by optimizing the loss $\mathcal{L}_{InfoNCE}$.

\subsection{ContraBAR Algorithm}\label{subsection:algorithm}
We will now introduce our CPC based meta RL algorithm, which is depicted in \cref{fig:contrabar_arch}, and explain how CPC is used to learn a latent representation of the history.

We begin by noting that we use the term observation throughout the text, in line with \citet{oord2018representation}, however in our case the meaning is state, reward and action $o_{t} = \{s_{t},r_{t-1},a_{t-1}\}$ when talking about ``observation history", and state and reward $o_{t+k}=\{s_{t+k},r_{t+k-1}\}$ when talking about ``future observations". We would like to learn an embedding of the observation history, $c_t$, that will contain relevant information for decision making. The CPC formulation seems like a natural algorithm to do this -- for a given trajectory $\tau$ of length $T$ collected from some unknown MDP $\mathcal{M}$, we use an embedding of its observation history until time $t < T$,  $c_t$, to learn to discriminate between future observations from the trajectory $\tau$ and random observations from other trajectories $\tau_{j} \neq \tau$. This means that $c_t$ encodes relevant information for predicting the future system states, and consequently information regarding the MDP $\mathcal{M}$ from which $\tau$ was collected. 

The CPC formulation described above is based on predicting future states in an uncontrolled system without rewards. We now modify it to learn a sufficient statistic for meta RL. We assume that data is collected at each training iteration $m$ by some data collection policy $\pi_m$ and added to a replay buffer $\mathcal{D}=\{\tau_i\}_{i=1}^{N}$ containing trajectories from previous data collection policies $\{\pi_1,\dots,\pi_{m-1}\}$; we note that length of the trajectories may vary. {At each learning iteration, a batch of $M$ trajectories is sampled, and for each trajectory and time $t$ the negative observations are sampled from the remaining $M-1$ trajectories in the batch. As in CPC, we define $c_t$ to be a function of the observation history until time $t$, but we add to $f$ as input the future $k-1$ actions, as in a controlled system the future observation $o_{t+k}=(s_{t+k},r_{t+k-1})$ depends on the controls. Our $f$ can therefore now be written as
$f(c_t, o_{t+k},a_{t:t+k-1})$. We implement this modification as in \citet{guo2018neural} by means of an additional autoregressive component, a GRU $g_{action}$ that receives actions as input and takes $c_t$ as its initial hidden state.

Given the adjustments described above, each batch $B$ used as input to our algorithm contains the following:
\begin{enumerate*}[label=(\arabic*)]
    \item{ The observation history until time $t$ in some trajectory $\tau$}
    \item Future observations from time $t+k$ 
    \item Observations $o_{t+k}^{-}$ from the remaining $M-1$ trajectories sampled from $\mathcal{D}$
\end{enumerate*}.

We rewrite the InfoNCE loss for the meta RL setting explicitly; For ease of notation we mark $f(c_t, o_{t+k}^{+},a_{t:t+k-1})$ as $f^{+}$ and $f(c_t, o_{t+k}^{-,i},a_{t:t+k-1})$ as $f^{-,i}$: 
\begin{equation}\label{eq:metainfonce}
\mathcal{L}_{M}=
\mathbb{E}_{B}\left[\log\frac{\exp (f^{+})}{\exp(f^{+}) + \sum_{i=1}^{K} \exp \left(f^{-,i}\right)}\right],
\end{equation}
 
 where the expectation is over the batches of positive and negative observations sampled from $\mathcal{D}$, as described above.

\subsection{Learning Information States with CPC}\label{subsection:theory}
We now show that integrating contrastive learning with meta RL is a fundamentally sound idea. We shall prove that our algorithm presented in \cref{subsection:algorithm} learns a representation of the history that is an information state, 
by showing that the latent encoding satisfies the properties of an information state \textbf{P1, P2} as defined by \citet{subramanian2020approximate} and reiterated above in \cref{section:background}. 

We first define the notion of a ``possible'' history, which we use in Assumption \ref{ass:discrete_action}.
\begin{definition}[Possible history]
\label{def:possible_hist}
Let $P_M$ denote a probability distribution over MDPs and let $P_{m,\pi}(h_t)$ be the probability of observing history $h_t$ under policy $\pi$ in MDP $m$. We say that $h_t$ is a possible history if there exists a policy $\pi$ and an MDP $m$ such that $P_M(m) > 0$ and $P_{m,\pi}(h_t) > 0$.
\end{definition}
Next, we make the following assumption, which states that the policy collecting the data covers the state, reward and action space.
\begin{assumption}\label{ass:discrete_action}
    Let the length of the longest possible history be $T$. Let $h_t$, where  $t \leq T$, be a history and let $P_{\mathcal{D}}(h_t)$ denote the probability of observing a history in the data $\mathcal{D}$. If $h_t$ is a possible history, then $P_{\mathcal{D}}(h_t)>0$.
\end{assumption}
Assumption \ref{ass:discrete_action} is necessary to claim that the learned CPC representation is a sufficient statistic \textit{for every} possible history. In Section \ref{sec:approx_info_state} we discuss a relaxation of this assumption, using \textit{approximate} information states.
\begin{theorem}\label{thm_main}
Let Assumption \ref{ass:discrete_action} hold. Let $g^{*},g^{*}_{AR},f^{*}$ jointly minimize $ \mathcal{L}_{M}(g,g_{AR},f) $. Then the context latent representation
$c_t=g^{*}_{AR}(z_{\leq t})$ satisfies conditions \textbf{P1, P2} and is therefore an information state.
\end{theorem}

 The full proof is provided in \cref{appendix:thmproofs}; we next provide a sketch. The main challenge in our proof lies in proving the following equality: 
\begin{equation}\label{eq:main_result}
P(s_{t+1},r_t|h_t,a_t) = P(s_{t+1},r_t|c_t,a_t).
\end{equation} 
Given the equality in \cref{eq:main_result}, proving \textbf{P1,P2} is relatively straightforward. We prove \cref{eq:main_result} by expanding the proof in \citep{oord2018representation}, which shows that the InfoNCE loss upper bounds the negative mutual information between $o_{t+k}^{+}$ and $c_t$ (in the CPC setting). In our case, we show that 
\begin{equation}\label{eq:CPC_equals_MI_bounds}
\mathcal{L}_{M} \geq \log(M-1) - I(s_{t+1},r_t;c_t|a_t),
\end{equation}
where $I(\cdot; \cdot)$ denotes mutual information. Thus, by minimizing the loss in \cref{eq:metainfonce}, we maximize the mutual information $I(s_{t+1},r_t;c_t|a_t)$. Due to the Markov property of the process, the mutual information in \eqref{eq:CPC_equals_MI_bounds} cannot be greater than $I(s_{t+1},r_t;h_t|a_t)$, which leads to Equation \ref{eq:main_result}.  

\subsection{Learning Approximate Information States with CPC}\label{sec:approx_info_state}

We next investigate a more practical setting, where there may be errors in the CPC learning, and the data does not necessarily satisfy Assumption \ref{ass:discrete_action}. We aim to relate the CPC error to a bound on the suboptimality of the resulting policy.
In this section, we consider an iterative policy improvement algorithm with a similarity constraint on consecutive policies, similar to the PPO algorithm we use in practice~\cite{schulman2017proximal}. We shall bound the suboptimality of policy improvement, when data for training CPC is collected using the previous policy, denoted $\pi_k$. 

In light of Eq.~\ref{eq:CPC_equals_MI_bounds}, we assume the following error due to an imperfect CPC representation:
\begin{assumption}\label{assumption:eps}
    There exists an $\epsilon$ such that for every $t\leq T$,
$I(s_{t+1},r_t;c_t|a_t) \geq I(s_{t+1},r_t;h_t|a_t) - \epsilon,$ where the histories are distributed according to policy $\pi_k$.
\end{assumption}




The next theorem provides our main result.
\begin{theorem}\label{thm:approx_ais}
Let Assumption \ref{assumption:eps} hold for some representation $c_t$. Consider the distance function between two distributions $D(P_1(x),P_2(x)) = \max_x | P_1(x)/P_2(x) |$.
We let $\hat{r}(c_t,a_t)=\mathbb{E}[r_t|c_t,a_t]$ and $\hat{P}(c'|c_t,a_t)=\mathbb{E}[\mathbf{1}(c_{t+1}=c')|c_t,a_t]$ denote an approximate reward and transition kernel, respectively. Define the value functions
\begin{equation}\label{eq:DP_inside_thm}
    \begin{split}
        \hat{Q}_t(c_t, a_t) &= \hat{r}(c_t, a_t) + \sum_{c_{t+1}}\hat{P}(c_{t+1}|c_t,a_t) \hat{V}_{t+1}(c_{t+1}) \\
        \hat{V}_t(c_t) &= \max_{\pi: D(\pi(c_t), \pi_k(c_t)) \leq \beta} \sum_{a} \pi(a) \hat{Q}_t(c_t, a),
    \end{split}
    \raisetag{2.2em}
\end{equation}
for $t\leq T$, and $\hat{V}_T(c_T) = 0$, and
the approximate optimal policy 
\begin{equation}\label{eq:policy_opt_inside_thm}
\hat{\pi}(c_t)\in \argmax_{\pi: D(\pi, \pi_k(c_t)) \leq \beta} \sum_{a} \pi(a) \hat{Q}_t(c_t, a).
\end{equation}
Let the optimal policy $\pi^*(h_t)$ be defined similarly, but with $h_t$ replacing $c_t$ in \eqref{eq:DP_inside_thm} and \eqref{eq:policy_opt_inside_thm}.
Then we have that 
\begin{equation*}
\begin{split}
    &\mathbb{E}^{\pi^*} \left[ \sum_{t=0}^{T-1} r(s_t,a_t) \right] - \mathbb{E}^{\hat{\pi}} \left[ \sum_{t=0}^{T-1} r(s_t,a_t) \right] \leq \\
    & \epsilon^{1/3} R_{max} T^2 (\sqrt{2} + 4\beta^T).
\end{split}
\end{equation*}
\end{theorem}
The dynamic programming recurrence in Equation \eqref{eq:DP_inside_thm} defines the optimal policy that is conditioned on $c_t$ (and not $h_t$), and is restricted to be $\beta$-similar to the previous policy $\pi_k$. The theorem bounds the loss in performance of such a policy compared to a policy that is conditioned on the full history (yet still restricted to be $\beta$-similar to $\pi_k$). The proof of Theorem \ref{thm:approx_ais} builds on the idea of an approximate information state~\citep{subramanian2020approximate} and is detailed in Appendix \ref{appendix:thmproofs}.

\subsection{ContraBAR Architecture}\label{section:contrabar_arch}

We now describe several design choices in our ContraBAR implementation.

\paragraph{History Embedding}
We now describe the specific architecture used to implement our algorithm, also depicted in \cref{fig:contrabar_arch}. We use a non-linear encoder to embed a history of actions, rewards and states and run it through a GRU to generate the hidden state for the current time-step $c_t$. The latent $c_t$ is then used to initialize the action-gru $g_{action}$, which is fed future actions as input -- the resulting hidden state is then concatenated with either a positive observation-reward pair, or a negative one and used as input to a projection head that outputs a score used in \cref{eq:metainfonce}.

We note that given a random sampling of negative observations, the probability of sampling a positive and negative observation that share the same state is low. Consequently, for environments where $s_{t+k}$ can be estimated via $s_t,a_t,\dots, a_{t+k}$ without $h_t$,  $c_t$ need only encode information regarding $s_t$ to allow the action-gru to learn to distinguish between positive and negative observations. This renders $c_t$ uninformative about the reward and transition functions and thus unhelpful for optimal control. An example of this is a set of deterministic environments that differ only in reward functions. The action-gru can learn to predict $s_{t+k}$ via $s_t,a_t,\dots, a_{t+k}$, only requiring $c_t$ to encode information regarding $s_t$ and \textit{not} the reward function. One way to circumvent this is hard negative mining, i.e using negative samples that are difficult to distinguish from the positive ones. Another solution, relevant for the case of varying reward functions, is to generate a negative observation by taking the state and action from the positive observation and recalculating the reward with a reward function sampled from the prior. In practice, we found that a simple alternative is to omit the action-gru. This prevents the easy estimation of $s_{t+k}$ and requires $c_t$ to encode information regarding the reward  and transition function. We found this worked well in practice for the environments we ran experiments on, including those with varying transitions. We expand on these considerations in Appendix \ref{appdx:arch_details}.

\paragraph{RL Policy}
 The history embedding portion of the algorithm described above is learned separately from the policy and can be done online or offline. The policy, which can be trained with an RL algorithm of the user's choice, is now conditioned on the current state $s_t$ as well as $c_t$ -- the learned embedding of $h_t$. We chose to use PPO \citep{schulman2017proximal} for the online experiments and SAC \citep{haarnoja2018soft} for the offline experiment -- in line with VariBAD and BOReL \citep{dorfman2021offline}.

\section{Experiments}

In our experiments, we shall demonstrate that
\begin{enumerate*}[label=(\arabic*)]
\item ContraBAR learns approximately Bayes-optimal policies
\item ContraBAR is on par with SOTA for environments with state inputs
\item ContraBAR scales to image-based environments
\item Augmentations can be naturally incorporated into ContraBAR and
\item ContraBAR can work in the offline setting
\end{enumerate*}.

We compare ContraBAR to state-of-the-art approximately Bayes-optimal meta RL methods. In the online setting, we compare against VariBAD~\citep{zintgraf2020varibad}, RL$^2$~\citep{duan2016rl}, and the recent modification of RL$^2$ by \citet{ni2022recurrent} which we refer to as RMF (recurrent model-free).
In the offline setting, we compare with BOReL~\citep{dorfman2021offline}. 
\citet{zintgraf2020varibad} and \citet{dorfman2021offline} already outperform posterior sampling based methods such as PEARL~\citep{rakelly2019efficient}, therefore we do not include such methods in our comparison.
Finally, we note that using VariBAD \citep{zintgraf2020varibad} with image-based inputs is currently computationally infeasible due to memory constraints, and as such we did not use it as a baseline -- we explain this issue further in \cref{appendix:varibad_comp}. Other variational approaches, which require a reconstruction of the future observations, are subject to similar memory constraints. Instead, we compared our algorithm against RL$^2$~\citep{duan2016rl}, which works with images. We evaluate performance similarly to \citet{zintgraf2020varibad}, by evaluating per episode return for 5 consecutive episodes with the exception of the offline setting where we adapted our evaluation to that of BOReL.

\subsection{Qualitative Near Bayes-Optimal Behavior}

We begin with a qualitative demonstration that ContraBAR can learn near Bayes-optimal policies. As calculating the exact Bayes-optimal policy is mostly intractable, we adopt the approach of \citet{dorfman2021offline}: for deterministic domains with a single sparse reward, the Bayes-optimal solution is essentially to search all possible reward locations so as to maximally reduce uncertainty, and then go directly to the goal in subsequent episodes. Thus, we can identify whether a policy is approximately Bayes-optimal by inspecting its trajectory.
Figure \ref{fig:qualitative} displays rollouts from a trained policy in the Gridworld and Semi-Circle domains, demonstrating near Bayes-optimal behavior similar to VariBAD \citep{zintgraf2020varibad}.

\begin{figure}
     \centering
     \begin{subfigure}[b]{0.45\columnwidth}
         \centering
         \includegraphics[width=\textwidth]{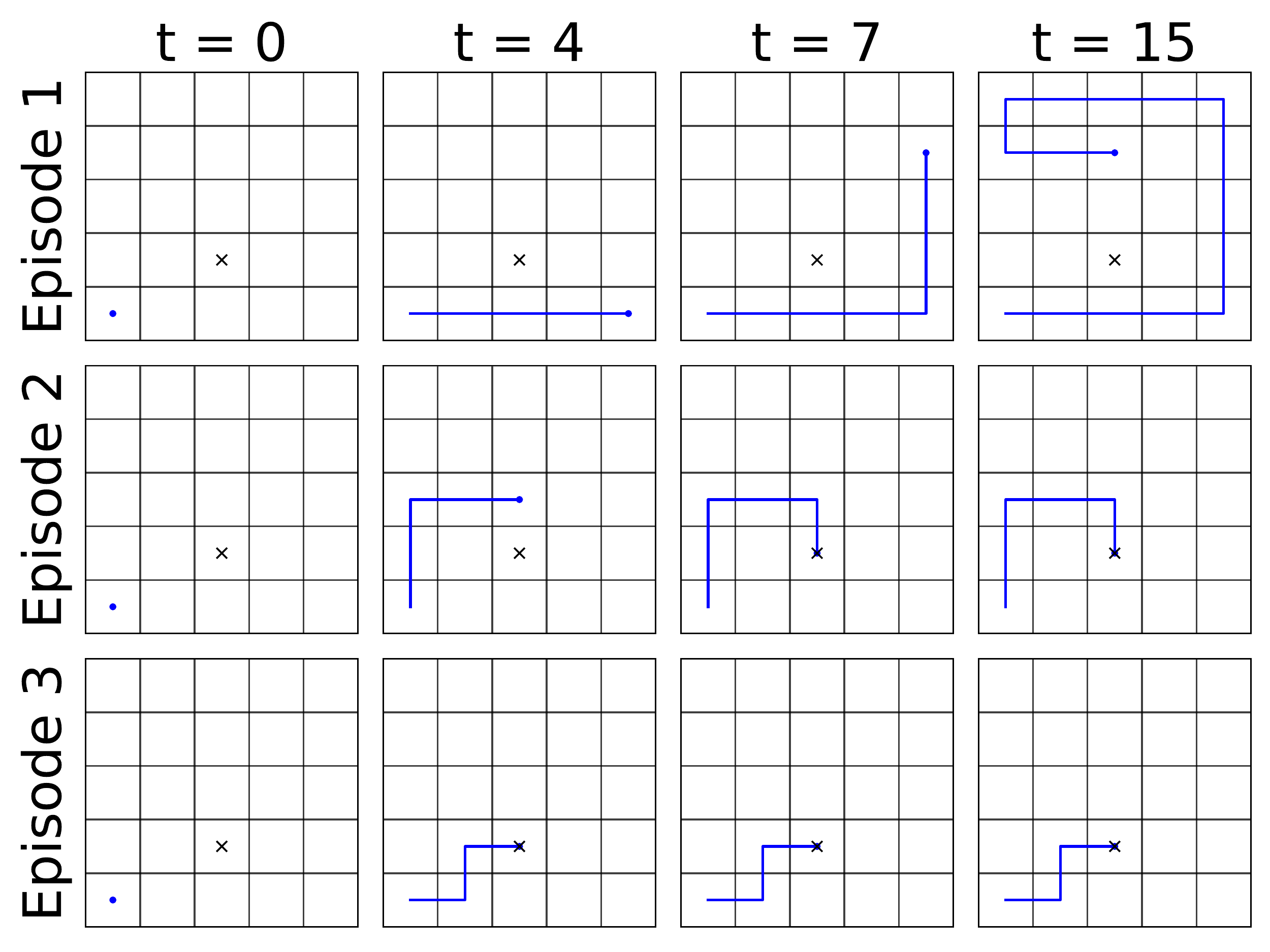}
         
     \end{subfigure}
     \begin{subfigure}[b]{0.45\columnwidth}
         \centering
         \includegraphics[width=\textwidth]{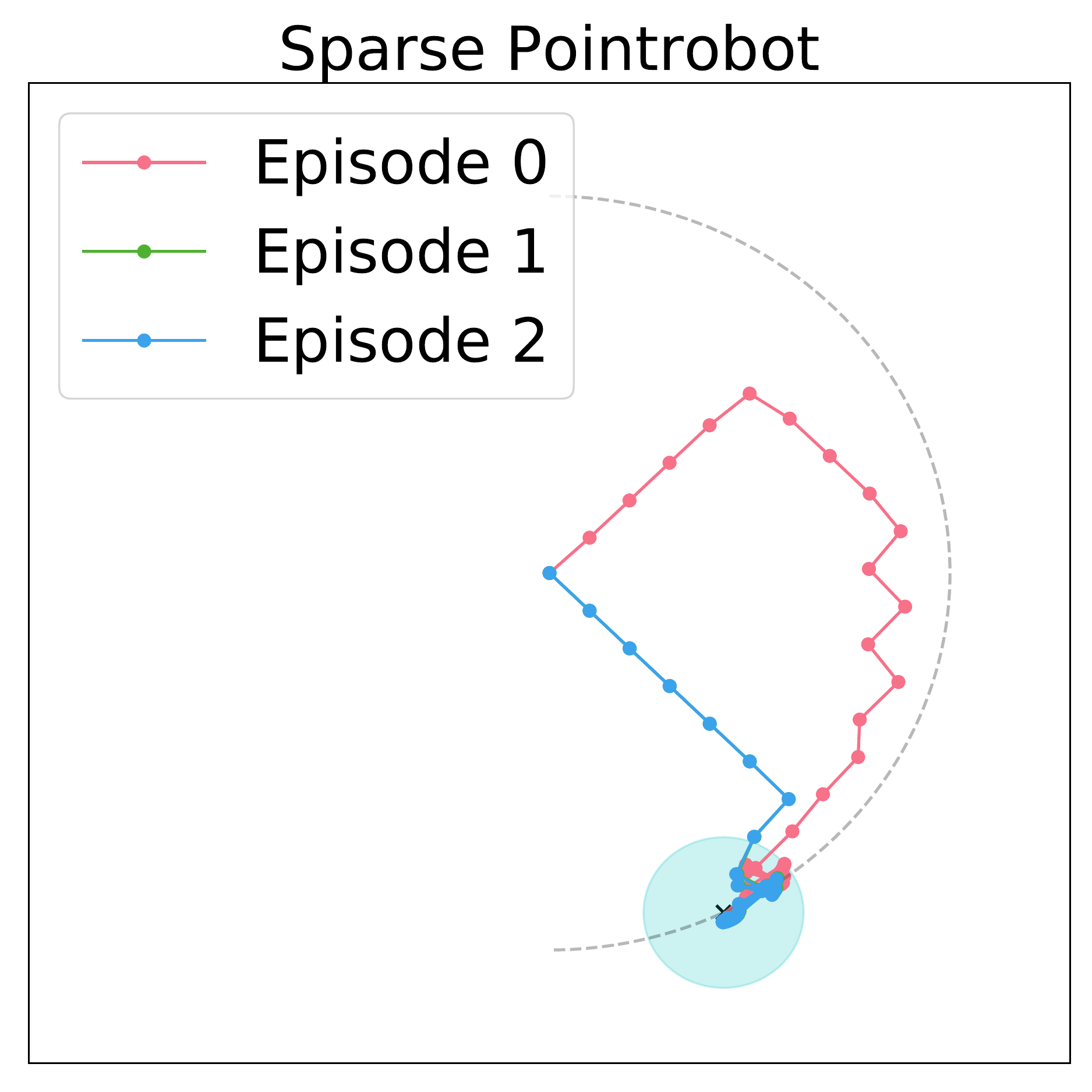}
     \end{subfigure}\caption{\textbf{Left:} Qualitatively Bayes-Optimal behavior on gridworld. The agent methodically searches the grid, reducing uncertainty about unexplored cells in the second episode and updating to a shorter path in the subsequent episode. \textbf{Right:} Qualitatively Bayes-Optimal behaviour on Sparse Semi-Circle. The agent scans the semi-circle until the goal has been found and all uncertainty reduced. In subsequent episodes the agent goes straight to the goal.} 
     \label{fig:qualitative}
     \vspace{-2em}
\end{figure}





\subsection{Results for Problems with State Observations}\label{subsection:sota_perf}

We compare ContraBAR with VariBAD and RMF, the current state-of-the-art on MuJoCo locomotion tasks~\cite{todorov2012mujoco}, commonly used in meta RL literature. We use the environments considered in \citet{zintgraf2020varibad}, namely the Ant-Dir, AntGoal, HalfCheetahDir, HalfCheetahVel, Humanoid and Walker environments.  \cref{fig:mujoco_test_perf} shows competitive performance with the current SOTA on all domains. Note that rewards in
these environments are dense, so in principle, the agent only needs a few exploratory
actions to infer the task by observing the rewards it receives. Indeed, we see that ContraBAR is able to quickly adapt within the first episode, with similar performance in subsequent episodes.

\begin{figure*}[t]
  \centering
  \vskip 0.2in
  \begin{center}
\centerline{\includegraphics[clip, trim=1.5cm 1.5cm 1cm 1.5cm,width=0.8\linewidth]
{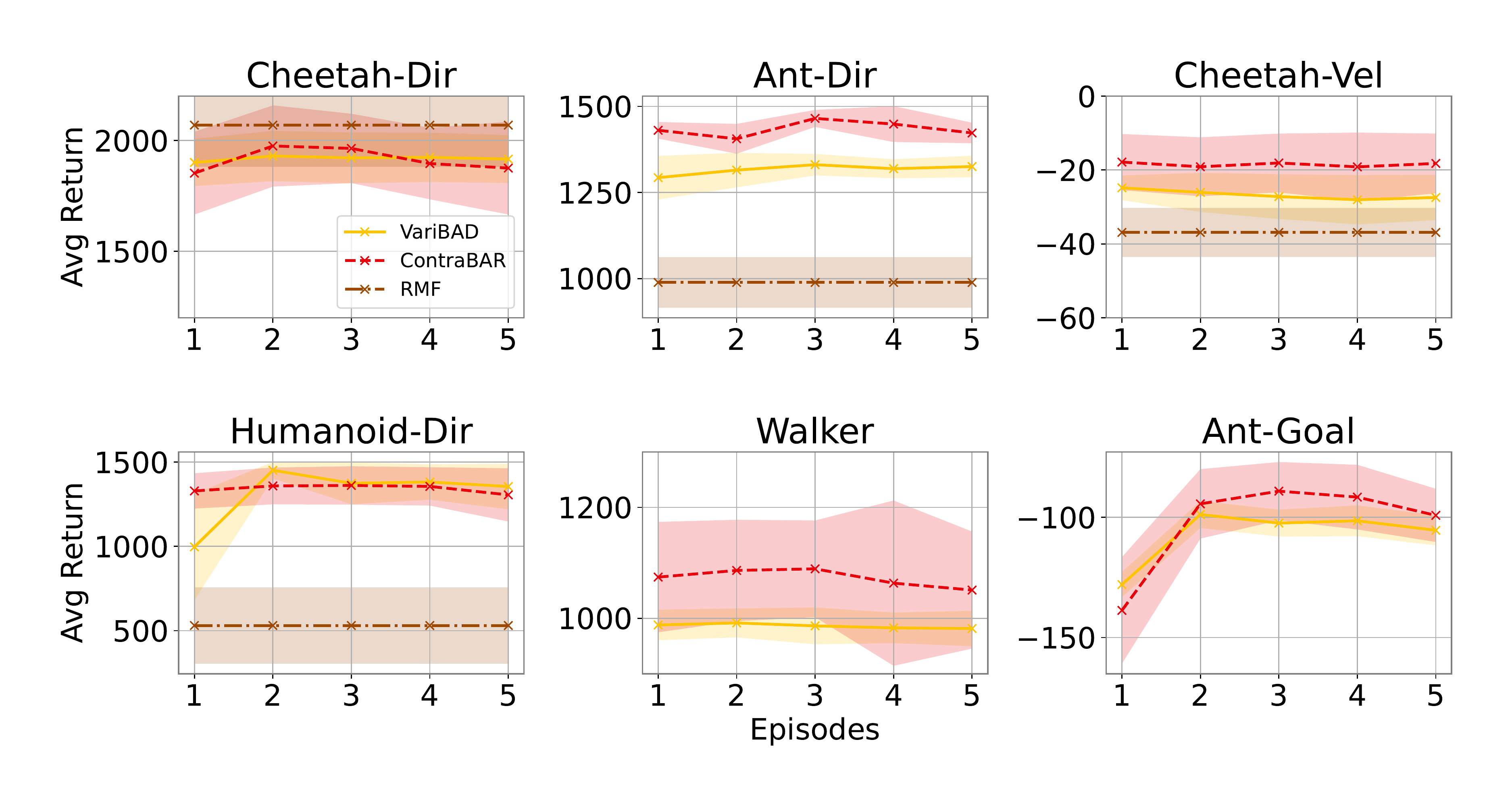}}
\vspace{-1em}
  \caption{Average test performance on six different MuJoCo environments, trained separately with 10 seeds per MuJoCo environment per method - as in \citep{zintgraf2020varibad}. The meta-trained policies are rolled out for 5 episodes to show how they adapt to the task. Values
shown are averages across tasks (95\% confidence intervals shaded). We show competitive performance, surpassing VariBad and RMF~\citep{ni2022recurrent} on certain environments. The results shown for RMF are the average performance on the first 2 episodes, as taken from their repository, which did not contain results for more than 2 episodes and did not contain results for the Walker and Ant-Goal environments. }\label{fig:mujoco_test_perf}
\end{center}
\vspace{-3em}
\end{figure*}

\begin{figure*}[!h]
   \centering
  \vskip 0.2in
  \begin{center}
  \centerline{\includegraphics[clip, trim=1.5cm 1.5cm 1cm 1.5cm, width=0.65\textwidth]{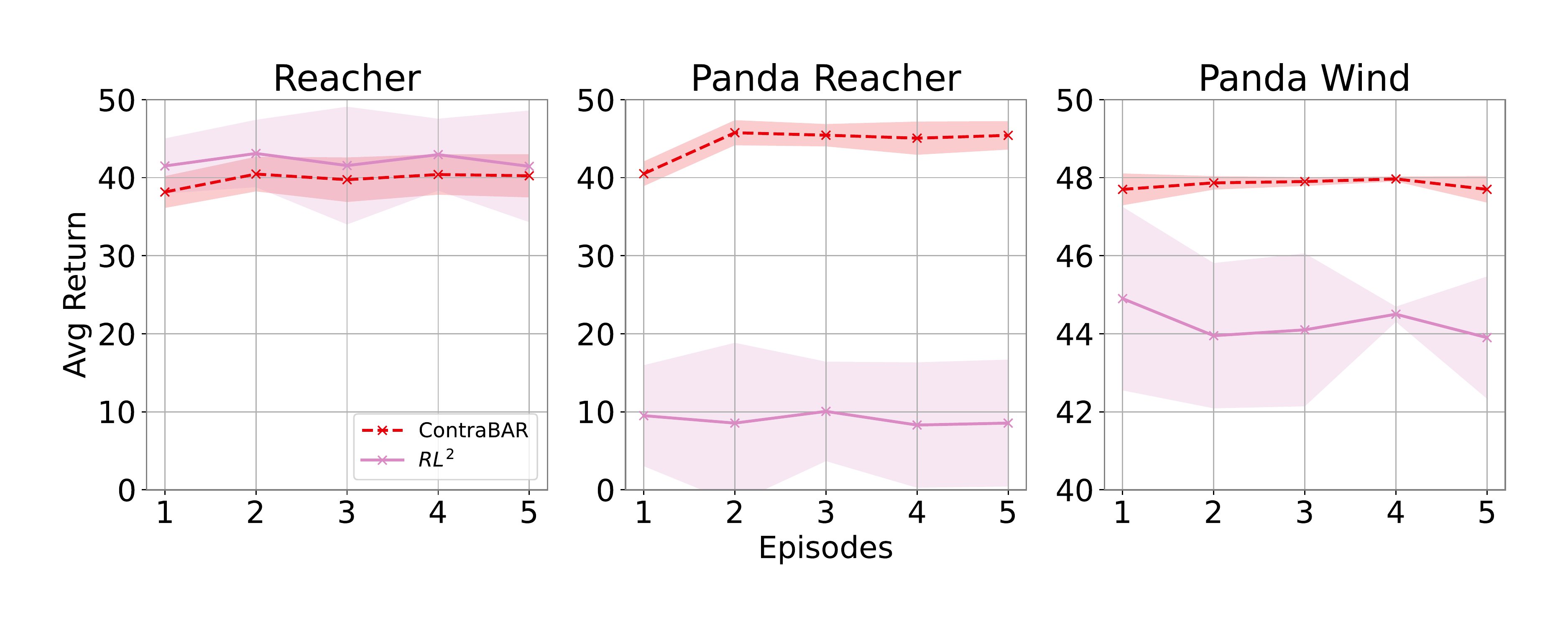}}
  \vspace{-1em}
  \caption{Average test performance on three different image-based environments, trained separately with different seeds per environment per method. The meta-trained
policies are rolled out for 5 episodes to show how they adapt to the task. Values
shown are averages across tasks (95\% confidence intervals shaded). All methods were run with 5 seeds for each environment. We show competitive performance with RL$^2$, surpassing it in the Panda environments.}\label{fig:image_perf}
  \end{center}
\vskip -0.2in
\end{figure*}

\begin{figure}[!h]
  \centering
  \vskip 0.2in
  \begin{center}
\centerline{\includegraphics[width=0.5\textwidth]  {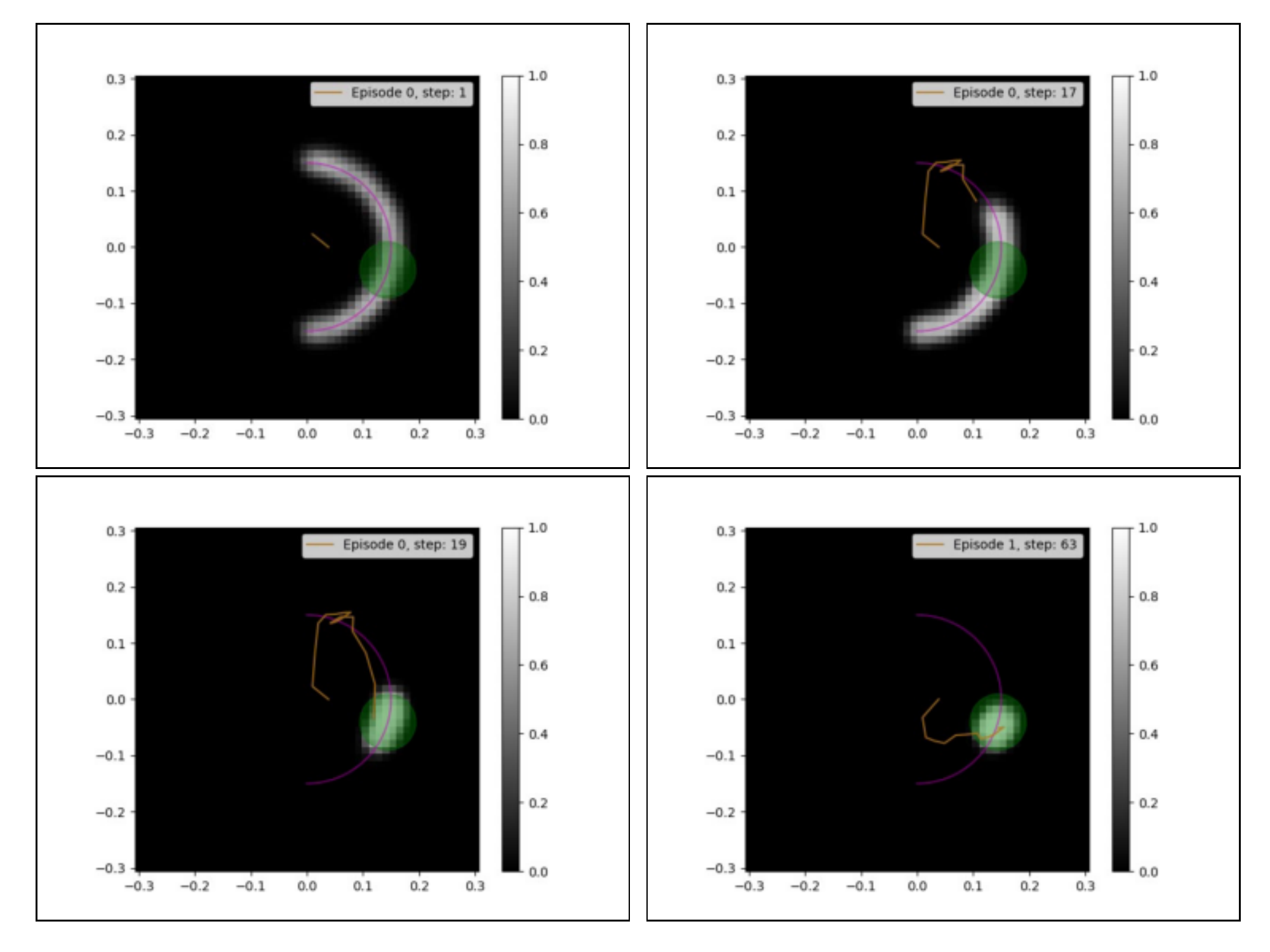}}
\vspace{-1em}
  \caption{Belief visualization on the Panda Reacher environment. At each step, we visualise the predictions of the MLP-classifier for the entire map. Each frame shows an interesting milestone along the trajectory. From left to right along the rows: (1) Initially, the agent's belief over the location of the reward is uniform over the semi-circle, as in the prior. (2) As the agent travels along the semi-circle, the uncertainty regarding the location of the goal collapses in locations where the agent has visited and not found a goal.
 (3) Once the agent reaches the goal, all uncertainty collapses, even for locations the agent has yet to visit.
(4) On the second episode, the agent follows the maintained belief and goes directly to the goal without exploration.
 }\label{fig:belief_panda}
  \end{center}
\vspace{-3em}
\end{figure}

\begin{figure}[t]
  \centering
\includegraphics[width=0.65\columnwidth]  {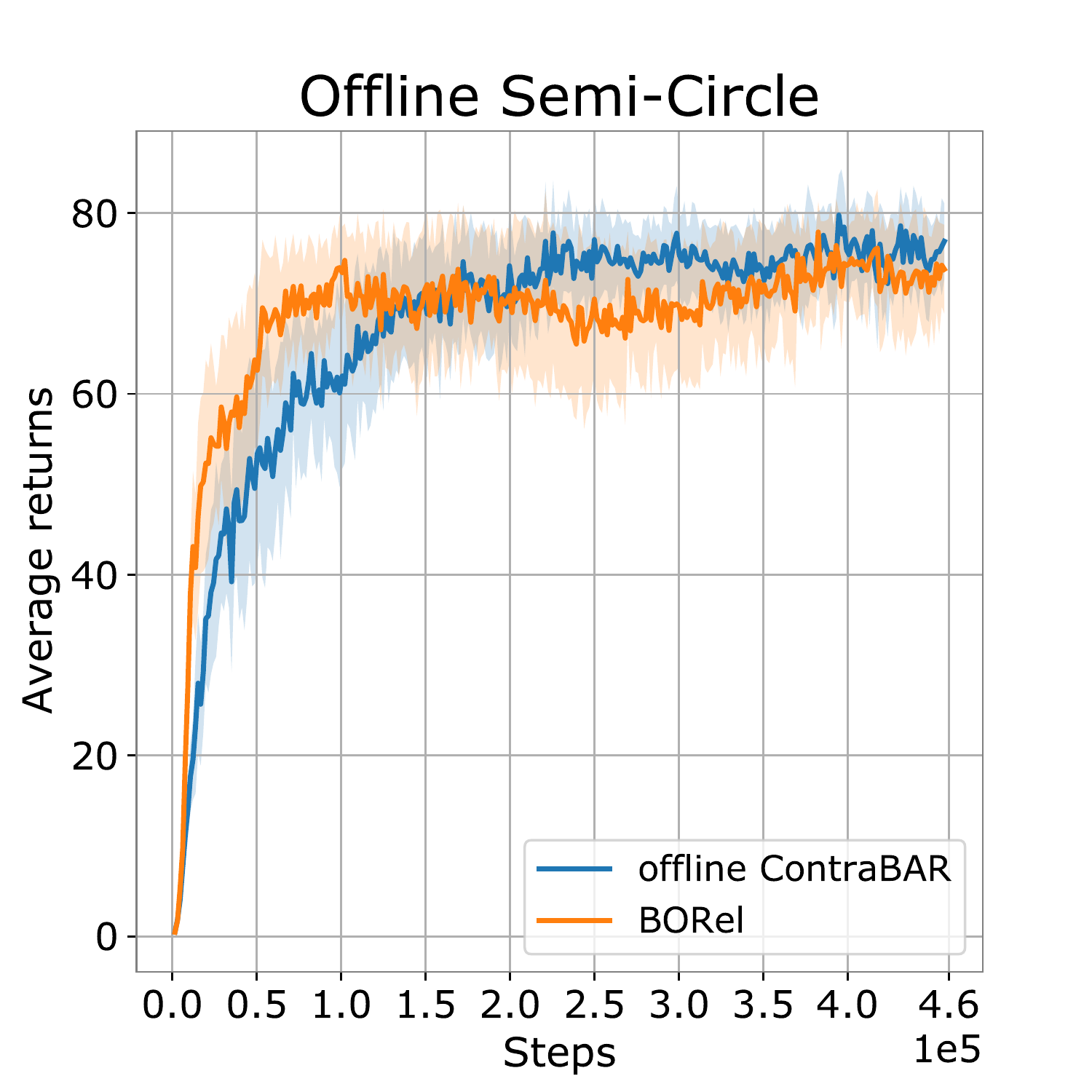}
\vspace{-1em}
  \caption{Results for the sparse Semi-Circle environment in the offline setting. Values shown are the average returns over two episodes, averaged over tasks. The shaded areas indicate 95\% confidence intervals over different seeds of the method. ContraBAR and BOReL were both run with 10 seeds with reward relabeling.
  }\label{fig:offline_contrabar}
  \vspace{-2em}
\end{figure}

\subsection{Scaling Belief to Image-Based Inputs}
We show that ContraBAR can scale to image domains, which are computationally expensive, by running our algorithm on three image-based domains with varying levels of difficulty and sources of uncertainty:  \begin{enumerate*}[label=(\arabic*)]
\item Reacher-Image -- a two-link
robot reaching an unseen target located somewhere on the diagonal of a rectangle, with sparse rewards
\item Panda Reacher -- a Franka Panda robot tasked with placing the end effector at a goal on a 2d semi-circle, where the vertical position of the goal ($z$ coordinate) is fixed; adapted from the Reacher task in Panda Gym \citep{gallouedec2021pandagym}
\item Panda Wind -- The same environment as Panda Reacher, except that the transitions are perturbed with Gaussian noise sampled separately for each task.
\end{enumerate*} 
For a more detailed description of each environment see \cref{appendix:environments}.  
 
\begin{figure}[h!]
   \centering
  \vskip 0.2in
  \begin{center}
  \centerline{\includegraphics[clip, trim=0.8cm 0.8cm 0.8cm 0.8cm,width=0.9\columnwidth]{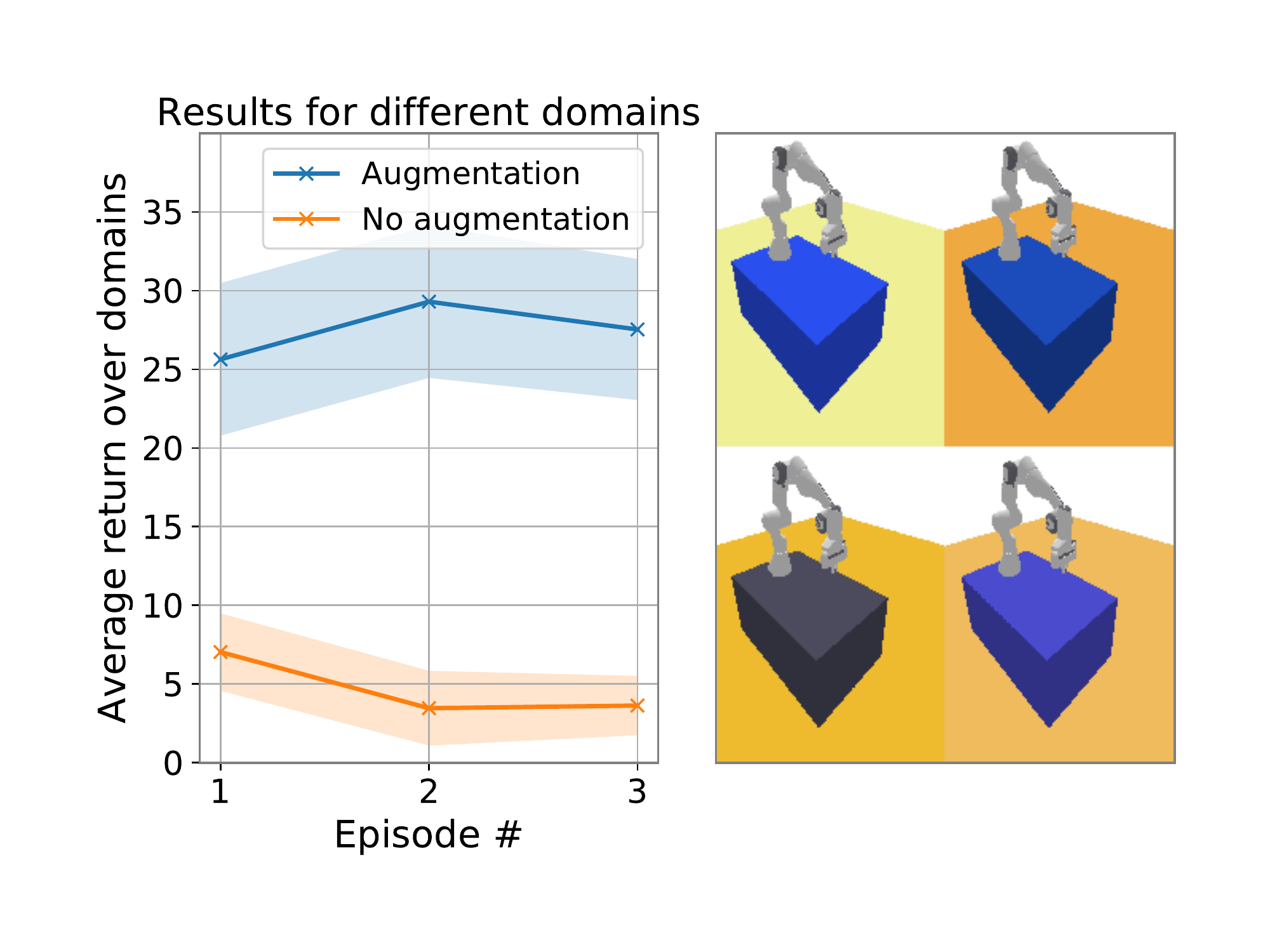}}
  \vspace{-1em}
  \caption{\textbf{Left:} Average test performance on three Panda Reacher environments with different color schemes from the training data.  We show two meta-trained
policies, one with augmentations and one without, each rolled out for 3 episodes to show how they adapt to the task. Values
shown are averages across tasks (95\% confidence intervals shaded). The agent trained with augmentations is invariant of the color schemes. \textbf{Right:} Top left image is the training environment; others are different color schemes for evaluation.}\label{fig:augment_perf}
  \end{center}
\vskip -0.2in
\end{figure}

\paragraph{Image-Based Reward:} For our image-based experiments, we found that learning in image-based domains with sparse reward was difficult when the reward was embedded separately (as in the state observation domains), and concatenated with the image embedding. We hypothesized that this might be an issue of differing scales between the scalar rewards and image inputs, but we observed that standard normalization techniques such as layer norm \citep{ba2016layer} did not help. Instead, we opted for a different approach that embeds the reward as an explicit part of the image. To implement this idea, we exploited the fact that in all our domains, the reward is sparse and binary, and we add a colored strip to a fixed place in the image when non-zero reward is received. Extending this idea to non-binary reward is possible, for example, by controlling the color of the strip.

Our results are displayed in \cref{fig:image_perf}. For the Reacher environment, ContraBAR is slightly outperformed by RL$^2$, whereas in Panda Reacher and Panda Reacher Wind ContraBAR outperform RL$^2$ by a large margin. Notice that in contrast to the dense reward domains of Section \ref{subsection:sota_perf}, in these sparse reward tasks the agent gains by exploring for the goal in the first iteration. Evidently, the plots show significantly higher reward in the second episode onward.

\paragraph{Glass-box Approach}
To further validate that our algorithm learns a sound belief representation, we follow a glass-box approach similar to that of \citet{guo2018neural}.
First, we used ContraBAR to learn an information state for the Panda Reacher environment. Second, we use the trained agent to create a dataset of trajectories, including the agent's belief at each time step of every trajectory. We then trained an MLP-based binary classifier, which takes $(x,y)$ and the information state $c_t$ as input and predicts whether the goal in the trajectory is indeed $(x,y)$. In \cref{fig:belief_panda} we see the visualization of the classifier's prediction at different points along the trajectory; We see that the predictions coincide with the belief we expect the agent to hold at each step, thus validating the soundness of our belief representation.

\subsection{ContraBAR with Domain Randomization}
Despite the high fidelity of modern simulators, when deployed in the real-world, image-based algorithms learned in simulation can only be accurate
up to the differences between simulation and reality – the
sim-to-real gap. This motivates us to learn a belief representation that is robust to such differences, and in the following we will show that our algorithm can indeed learn such an information state. Robustification to irrelevant
visual properties via random modifications is termed \textit{domain randomization} \citep{tobin2017domain}. We employ domain randomization in a similar fashion to \citet{rabinovitz2021unsupervised} wherein we modify the past and future observations (without the rewards) in the trajectories with a mapping $\mathcal{T}: \mathcal{S} \rightarrow \mathcal{S}$ that randomly shifts the RGB channels of the images. These modified trajectories are used to learn the history embedding $c_t$, with the hope that it will be invariant to different color schemes in the environment. We show the strength of such modifications by training two agents with ContraBAR on the Panda Reacher environment -- one receives images modified by $\mathcal{T}$ and the other does not. We then evaluate each agent's performance on different color schemes, which are kept static for evaluation. The results as well as the environments can be seen in \cref{fig:augment_perf}.  Note that while the belief may be robustified separately with augmentations, the policy must be robust to such changes as well. To do so, we used the data-regularized actor-critic method from \citet{raileanu2021automatic} where the policy $\pi_{\theta}$ and value function $V_{\phi}$ are regularized via two additional loss terms,
\begin{equation*}
\begin{split}
      G_\pi&=KL\left [\pi_{\theta}(a |s) \vert \pi_{\theta}(a | T(s)) \right ],\\
  G_V&=\left(V_{\phi}(s)-V_{\phi}(T(s))\right)^2, 
\end{split}
\end{equation*}
where $T: \mathcal{S} \rightarrow \mathcal{S}$  randomly modifies the image.

We emphasize that domain randomization, as applied here, is not naturally compatible with variational belief inference methods. The reason is that when the loss targets \textit{reconstruction} of the modified observation, the learned embedding cannot be trained to be \textit{invariant} to the modification $\mathcal{T}$.

\subsection{Offline ContraBAR}
We show that as in VariBAD \citep{zintgraf2020varibad}, the disentanglement of belief and control allows us to reframe the algorithm within the context of offline meta RL, as was done in \citet{dorfman2021offline}. First, we use ContraBAR to learn a history embedding $c_t$ from an offline dataset. Note that no specific change is required to our algorithm -- we simply treat the offline dataset as the replay buffer for ContraBAR. Second, we perform \textit{state relabeling} as described in \citet{dorfman2021offline}: for each trajectory $\tau_i$ of length $T$, i.e $(s_0^{i},a_o^{i},r_0^{i},\dots,s_T^{i})$, we embed each partial $t$-length history $h_t$ as $c_t$, and transform each $s_t^{i}$ to $s_t^{+,i}=(s_t^{i},c_t^{i})$ as in the BAMDP formulation. We then learn a policy with SAC \citep{haarnoja2018soft} on the transformed dataset. We show competitive results with BOReL \citep{dorfman2021offline} in \cref{fig:offline_contrabar}.  Unfortunately we were not able to find an offline adaptation of RMF to use as an additional baseline.





\section{Conclusions}
We proved that ContraBAR learns a representation that is a sufficient static of the history.  Following on this, we presented what is to the best of our knowledge the first approximately Bayes-optimal CL meta RL algorithm. We demonstrated results competitive with previous approaches on several challenging state-input domains. Furthermore, by using contrastive learning we were able to scale meta-RL to image-based domains; We displayed results on par with RL$^2$ which was also able to scale to image inputs. Finally, we showed that our method is naturally amenable to domain randomization, which may be important for applications such as robotics.

\section{Acknowledgements}
We thank Tom Jurgenson, Ev Zisselman, Orr Krupnik and Gal Avineri for useful discussions and feedback, and Luisa Zintgraf for invaluable help with reproducing the graphs from the VariBAD paper. This work received funding from the European Union (ERC, Bayes-RL, Project Number 101041250). Views and opinions expressed are however those of the author(s) only and do not necessarily reflect those of the European Union or the European Research Council Executive Agency. Neither the European Union nor the granting authority can be held responsible for them. 

\newpage
\bibliographystyle{icml2023}
\bibliography{refs}

\appendix

\onecolumn

\section{Theorem proofs}\label{appendix:thmproofs}
\begingroup
\def\thetheorem{\ref{thm_main}}
\begin{theorem}
Let Assumption \ref{ass:discrete_action} hold. Let $g^{*},g^{*}_{AR},f^{*}$ jointly minimize $ \mathcal{L}_{M}(g,g_{AR},f) $. Then the context latent representation
$c_t=g^{*}_{AR}(z_{\leq t})$ satisfies conditions \textbf{P1, P2} and is therefore an information state.
\end{theorem}
\addtocounter{theorem}{-1}
\endgroup

We begin our proof by presenting the causal model for the variant of CPC used by ContraBAR, shown in \cref{fig:causal_model}.

\begin{figure}[!h]
   \centering
  \vskip 0.2in
  \begin{center}
  \centerline{\includegraphics[width=0.3\textwidth]{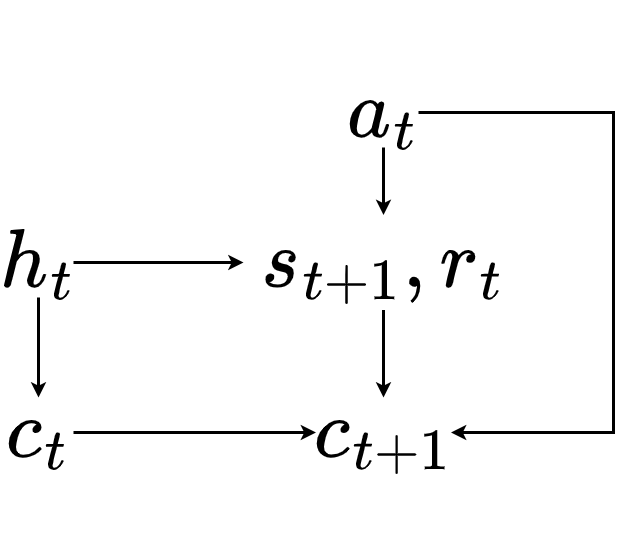}}
  \caption{The causal model for $c_t$}\label{fig:causal_model}

  \end{center}
\vskip -0.2in
\end{figure}

From the causal model, we can infer that
\[P(c_{t+1}|c_t,s_{t+1},r_t,a_t) =\] 
\[P(c_{t+1}|c_t,s_{t+1},r_t,a_t, h_t)\] and $P(s_{t+1},r_t|h_t, a_t) = P(s_{t+1},r_t|h_t, a_t, c_t)$. We shall also assume that $c_t$ is a deterministic function of $h_t$, and therefore $P(c_{t+1}|c_t,s_{t+1},r_t,a_t, h_t) = P(c_{t+1}|s_{t+1},r_t,a_t, h_t)$, and from the above, we have $P(c_{t+1}|s_{t+1},r_t, a_t, h_t) = P(c_{t+1}|s_{t+1},r_t,a_t, c_t)$.

We now prove a mutual information bound similar to that of \citet{oord2018representation}, we show that by optimizing the meta RL InfoNCE loss defined in \cref{eq:metainfonce} we maximize the mutual information between $c_t$ and $s_{t+1},r_t $ given $a_t$.

We begin with a lemma similar to that of Section 2.3 in \citet{oord2018representation}:
\begin{definition}[Possible sufficient statistic transition]
\label{def:possible_sst}
 Let $c_t$ be a function of $h_t$, i.e $c_t=\sigma_t(h_t)$. $s_{t+1},r_t,c_t,a_t$ is a possible sufficient statistic transition if $h_t, a_t, r_t, s_{t+1}$ is a possible history as in Definition \ref{def:possible_hist} and $c_t=\sigma_t(h_t)$.\end{definition}
\begin{lemma}\label
{lemma:optimal_f}
Let Assumption \ref{ass:discrete_action} and the loss in \cref{eq:metainfonce} be jointly minimized by $f, g, g_{AR}$, then for any possible sufficient statistic transition $s_{t+1},r_t,c_t,a_t$ as in Definition \ref{def:possible_sst}, where $c_t=g_{AR}(h_t)$, we have that \[f(s_{t+1},r_t,c_t,a_t) \propto \frac{P(s_{t+1},r_t|c_t,a_t)}{P(s_{t+1},r_t|a_t)}\].
\end{lemma}
\begin{proof}
The loss in Eq. \ref{eq:metainfonce} is the categorical cross-entropy of classifying the positive example correctly, with $\frac{f}{\sum_{B} f}$ being the prediction of the model. We denote the $j$-th example in the batch $B$ as $s_j,r_j$, where the subscript \textit{does not} refer to time here. As in \citep{oord2018representation}, the optimal probability for this loss is $P(d=i|B,c_t,a_t)$ (with $[d=i]$ indicating the i-th example in $B$ is the positive example) and can be derived as follows:

\begin{align}\label{eq:proof_cpc_1}
\begin{split}
P(d=i|B,c_t,a_t)&=\frac{P(s_{i},r_i|c_t,a_t) \Pi_{l\neq i} P(s_{l},r_l|a_t)}{\sum_{j=1}^{M}P(s_{j},r_j|c_t,a_t)\Pi_{l\neq j} P(s_{l},r_l|a_t)} \\
&=\frac{\frac{P(s_{i},r_i|c_t,a_t)}{P(s_{i},r_i|a_t)}}{\sum_{j=1}^{M}\frac{P(s_{j},r_j|c_t,a_t)}{P(s_{j},r_j|a_t)}}.
\end{split}
\end{align}
\end{proof}
Eq. \ref{eq:proof_cpc_1} means that for any $s_{t+1},r_t,c_t,a_t$ that are part of a batch $B$ in the data, we have that $f(s_{t+1},r_t,c_t,a_t) \propto \frac{P(s_{t+1},r_t|c_t,a_t)}{P(s_{t+1},r_t|a_t)}$. From Assumption \ref{ass:discrete_action}, for any sufficient statistic transition tuple $s_{t+1},r_t,c_t,a_t$ there exists a batch it is a part of. 

\begin{lemma}\label{lemma:extended_mi}
Let Assumption \ref{ass:discrete_action}, and let the loss in \cref{eq:metainfonce} be jointly minimized by $f,g,g_{AR}$. Then \[I(s_{t+1},r_t;c_t|a_t) \geq \log(M-1) - \mathcal{L}_{opt}.\]
\begin{proof}
 
Given the optimal value shown in Lemma \ref{lemma:optimal_f} for $f(s_{t+1}, r_t, c_t, a_t)$, by inserting back into the loss we get:
\begin{align*}
\begin{split}
&\mathcal{L}_{opt}=-\mathbb{E}\log\left[\frac{\frac{P(s_{t+1},r_t|c_t,a_t)}{P(s_{t+1},r_t|a_t)}}{\frac{P(s_{t+1},r_t|c_t,a_t)}{P(s_{t+1},r_t|a_t)} + \sum_{(s',r') \in \{o_{j}^{-}\}_{j=1}^{M-1}}\frac{P(s',r'|c_t,a_t)}{P(s',r'|a_t)}}\right] \\
&=\mathbb{E}\log \left [1 + 
(\frac{P(s_{t+1},r_t|a_t)}{P(s_{t+1},r_t|c_t,a_t)}
\sum_{(s',r') \in \{o_{j}^{-}\}_{j=1}^{M-1}}\frac{P(s',r'|c_t,a_t)}{P(s',r'|a_t)}\right ] \\
&\approx\mathbb{E}\log \left [1 + 
(\frac{P(s_{t+1},r_t|a_t)}{P(s_{t+1},r_t|c_t,a_t)}(M-1) \cdot \mathbb{E}_{\mathcal{D}(s',r'|a_t)}  \frac{P(s',r'|c_t,a_t)}{P(s',r'|a_t)} \right ] \\
&=\mathbb{E}\log \left [1 + 
(\frac{P(s_{t+1},r_t|a_t)}{P(s_{t+1},r_t|c_t,a_t)}(M-1) \right ] \\
& \geq\mathbb{E}\log \left [ 
\frac{P(s_{t+1},r_t|a_t)}{P(s_{t+1},r_t|c_t,a_t)}(M-1) \right ] \\
&=-I(s_{t+1},r_t;c_t|a_t)+\log(M-1).
\end{split}
\end{align*}

We therefore get that
\[I(s_{t+1},r_t;c_t|a_t)\geq \log(M-1) - \mathcal{L}_{opt}.\] 
We conclude that the objective maximizes the mutual information between $c_t$ and $s_{t+1},r_t$ given $a_t$.
\end{proof}
\end{lemma}

\begin{corollary}\label{corollary:maximally_info}
Let Assumption \ref{ass:discrete_action}, and let the loss in \cref{eq:metainfonce} be jointly minimized by $f,g,g_{AR}$, then $I(c_t; s_{t+1},r_t|a_t) = I(h_t; s_{t+1},r_t|a_t)$ where $I(\cdot; \cdot)$ denotes mutual information.
\begin{proof}
Since $s_{t+1},r_t$ depend only on $h_t$ (conditioned on $a_t$), and since $c_t$ is a deterministic function of $h_t$, $I(s_{t+1},r_t;c_t|a_t)$ cannot be greater than $I(s_{t+1},r_t;h_t|a_t)$. From Lemma \ref{lemma:extended_mi} , we therefore have that \[I(c_t; s_{t+1},r_t|a_t) = I(h_t; s_{t+1},r_t|a_t)\].
\end{proof}
\end{corollary}

Note that Corollary \ref{corollary:maximally_info} states that given the causal model above, $c_t$ is maximally informative about $s_{t+1},r_t$ (conditioned on $a_t$).
We use this result to prove a short lemma that will help us prove that that $c_t$ is an information state.
\begin{lemma}\label{lemma:c_equiv}
Let the assumptions of Corollary \ref{corollary:maximally_info} hold, then for every $a$, $P(s_{t+1},r_t|h_t,a_t) = P(s_{t+1},r_t|c_t,a_t)$.
\end{lemma}
\begin{proof}
We start with a result similar to the data processing inequality.
Consider $I(s_{t+1},r_t;h_t, c_t|a_t)$. We have that 
\begin{equation}\label{eq:data_proc_1}
\begin{split}
    &I(s_{t+1},r_t;h_t, c_t|a_t) = \\ &I(s_{t+1},r_t;c_t| h_t,a_t) 
    + I(s_{t+1},r_t   ;h_t|a_t),
\end{split}
\end{equation}
and on the other hand,
\begin{equation}\label{eq:data_proc_2}
\begin{split}
    &I(s_{t+1},r_t;h_t, c_t|a_t) = \\ &I(s_{t+1},r_t;h_t| c_t,a_t) + I(s_{t+1},r_t;c_t|a_t).
\end{split}
\end{equation}
From the causal graph above, we have that $I(s_{t+1},r_t;c_t| h_t,a_t) = 0$. Therefore, from Eq. \eqref{eq:data_proc_1} and \eqref{eq:data_proc_2} we have

\begin{equation*}
\begin{split}
    I(s_{t+1},r_t;h_t|a_t) &= I(s_{t+1},r_t;c_t|a_t) + I(s_{t+1},r_t;h_t| c_t,a_t) \\ &\geq I(s_{t+1},r_t;c_t|a_t)
    \end{split}
\end{equation*}

with equality only if $I(s_{t+1},r_t;h_t| c_t,a_t)=0$, since the mutual information is positive. From Corollary \ref{corollary:maximally_info}, we therefore must have $I(s_{t+1},r_t;h_t| c_t,a_t)=0$. This implies that $s_{t+1},r_t$ and $h_t$ are independent conditioned on $c_t,a_t$~\cite{coverelements}, and therefore \[P(s_{t+1},r_t|h_t,a_t) = P(s_{t+1},r_t|c_t,a_t)\].

\end{proof}
\begin{proposition}
Let Assumption \ref{ass:discrete_action}, and let the loss in \cref{eq:metainfonce} be jointly minimized by $f,g,g_{AR}$, then $c_t$ satisfies \textbf{P1}, i.e., $\mathbb{E}[r_t|h_t,a_t] = \mathbb{E}[r_t|c_{t},a_t]$.
\end{proposition}
\begin{proof}
\begin{equation*}
\begin{split}
\mathbb{E}[r_t|h_t,a_t] &= \int r_t \int P(s_{t+1}, r_t|h_t,a_t) ds_{t+1} dr_t \\
&= \int r_t \int P(s_{t+1},r_t|c_t,a_t) ds_{t+1} dr_t \\
&= \int r_t P(r_t | c_t, a_t) dr_t \\
&= \mathbb{E}[r_t|c_t, a_t].
\end{split}
\end{equation*}
\end{proof}
\begin{proposition}
Let Assumption \ref{ass:discrete_action} and let the loss in \cref{eq:metainfonce} be jointly minimized by $f,g,g_{AR}$, then $c_t$ satisfies \textbf{P2}, i.e., $P(c_{t+1}|h_t) = P(c_{t+1}|c_{t})$.
\end{proposition}
\begin{proof}

\begin{equation*}
\begin{split}
    &P(c_{t+1}|h_t,a_t) \\  &=\int \int P(s_{t+1},r_t|h_t,a_t)P(c_{t+1}|h_t, s_{t+1},r_t,a_t)ds_{t+1} dr_t \\
    &= \int \int P(s_{t+1},r_t|c_t,a_t)P(c_{t+1}|h_t, c_t, s_{t+1},r_t,a_t)ds_{t+1}dr_t \\
    &= \int \int P(s_{t+1},r_t|c_t,a_t)P(c_{t+1}|c_t, s_{t+1},r_t,a_t)ds_{t+1}dr_t \\
    &= P(c_{t+1}|c_t,a_t).
\end{split}
\end{equation*}

where the second equality is due to lemma \ref{lemma:c_equiv} and the penultimate equality is due to $c_{t+1}$ being a deterministic function of $c_t,s_{t+1},r_t$ and $a_t$
\end{proof}
We now provide the proofs for the setting described in Section \ref{sec:approx_info_state}, where there may be errors in the CPC learning, and the data does not necessarily satisfy Assumption \ref{ass:discrete_action}. 

We recapitulate that we consider an iterative policy improvement algorithm with a similarity constraint on consecutive policies, similar to the PPO algorithm we use in practice \cite{schulman2017proximal}. We shall bound the suboptimality of policy improvement, when data for training CPC is collected using the previous policy, denoted $\pi_k$.  We will show optimal policy bounds when the information state is approximate, similar in spirit to \citet{subramanian2020approximate}, but with additional technicalities. Under the setting above, we will bound the suboptimality in policy improvement in terms of an error in CPC training, which we denote $\epsilon$. 

In light of the bound from \ref{lemma:extended_mi}, we assume the following:
\begingroup
\def\thetheorem{\ref{assumption:eps}}
\begin{assumption}
    There exists an $\epsilon$ such that for every $t\leq T$,
$I(s_{t+1},r_t;c_t|a_t) \geq I(s_{t+1},r_t;h_t|a_t) - \epsilon,$ where the histories are distributed according to policy $\pi_k$.
\end{assumption}
\addtocounter{theorem}{-1}
\endgroup

We now define $P_{\pi}(h_t)$ as the probability of seeing a history under a policy $\pi$. For the sake of simplicity, for the subsequent section we will refer to $P_{\pi}(h_t)$ as $P(h_t)$. Furthermore, when the information state is approximate, we denote the information state generator $\hat{\sigma}_t$.

We begin with the following bound.
\begin{proposition}\label{proposition:dkl_eps} Let Assumption \ref{assumption:eps} hold, then
    $\mathbb{E}_{h_t \sim P(h_t)}
\left [D_{KL}(P(s_{t+1},r_t|h_t,a_t) || P(s_{t+1},r_t|\hat{\sigma}_t(h_t),a_t) \right ]\leq \epsilon$
\begin{proof}
\begin{proposition}\label{proposition:ib_epsilon}
Let Assumption \ref{assumption:eps} hold, then $I(s_{t+1},r_t;h_t|c_t,a_t) \leq \epsilon$
\begin{proof}

We start with a result similar to the data processing inequality.
We have that
\begin{align*}
    \begin{split}
        &I(s_{t+1},r_t;h_t,c_t|a_t)= \\
        &I(s_{t+1},r_t;c_t|h_t,a_t) + I(s_{t+1},r_t;h_t|a_t)
    \end{split}
\end{align*}
and,
\begin{align*}
    \begin{split}
        &I(s_{t+1},r_t;h_t,c_t|a_t)= \\
        &I(s_{t+1},r_t;h_t|c_t,a_t) + I(s_{t+1},r_t;c_t|a_t)
    \end{split}
\end{align*}
From the causal graph we have that $I(s_{t+1},r_t;c_t|h_t,a_t)=0$, yielding 
\begin{align*}
    \begin{split}
        &I(s_{t+1},r_t;h_t|a_t) = I(s_{t+1},r_t;h_t|c_t,a_t) + I(s_{t+1},r_t;c_t|a_t) \Rightarrow \\
        &I(s_{t+1},r_t;h_t|a_t) - I(s_{t+1},r_t;c_t|a_t) = I(s_{t+1},r_t;h_t|c_t,a_t)
    \end{split}
\end{align*}
Combined with \ref{assumption:eps} we get that $I(s_{t+1},r_t;h_t|c_t,a_t) \leq \epsilon$
\end{proof}
\end{proposition}
\textbf{We note that from here on out everything is conditioned on $a_t$, and omit it to avoid overly cumbersome notation.}

For ease of notation we define:
$z=s_{t+1},r_t$.
We note that given a specific $h_t$, we have:
\begin{align*}
    \begin{split}
        &D_{KL}\left (P_{z|h_t} || P_{z|\hat{\sigma}_t(h_t)}\right )=\int_{z} P(z|h_t) \cdot \log \left ( \frac{P(z|h_t)}{P(z|c_t)} \right ) 
    \end{split}
\end{align*}
\begin{proposition}\label{proposition:inf_dkl}
$I(z;h_t | c_t) = \mathbb{E}_{h_t \sim P(h_t)} \left [D_{KL}\left (P_{z|h_t} || P_{z|\hat{\sigma}_t(h_t)}\right ) \right ]$
\begin{proof}

\begin{align*}
    \begin{split}
        I(z;h_t|c_t)&=\mathbb{E}_{P_{\sigma_t(h_t)=c_t}} \left [D_{KL}\left ( P_{z,h_t|c_t} || P_{z|c_t}\cdot P_{h_t | c_t} \right ) \right]\\
        &=\mathbb{E}_{P_{\sigma_t(h_t)=c_t}} \left [\int_{h_t} \int_{z}P(z,h_t|c_t) \ \log \left ( \frac{P(z,h_t|c_t)}{P(z|c_t) \cdot P(h_t | c_t)}\right ) \right] \\
        &=\mathbb{E}_{P_{\sigma_t(h_t)=c_t}} \left [\int_{h_t} \int_{z}P(z,h_t|c_t) \ \log \left ( \frac{P(z,h_t,c_t)\cdot P(c_t)}{P(z,c_t) \cdot P(h_t, c_t)}\right ) \right] \\
        &=\mathbb{E}_{P_{\sigma_t(h_t)=c_t}} \left [\int_{h_t} \int_{z}\frac{P(z,h_t,c_t)}{P(c_t)} \ \log \left ( \frac{P(z|h_t)}{P(z|c_t) }\right ) \right] \\
        &= \int_{c_t}\int_{h_t} \int_{z}P(z,h_t,c_t)\ \log \left ( \frac{P(z|h_t)}{P(z|c_t) }\right ) \\
        &= \int_{c_t}\int_{h_t} \int_{z}P(z|h_t)P(h_t,c_t)\ \log \left ( \frac{P(z|h_t)}{P(z|c_t) }\right ) \\
        &= \int_{c_t}\int_{h_t} P(h_t,c_t)\int_{z}P(z|h_t)\ \log \left ( \frac{P(z|h_t)}{P(z|c_t) }\right ) \\
        &= \int_{h_t} P(h_t)\int_{c_t}\delta_{c_t=\sigma_t(h_t)}\int_{z}P(z|h_t)\ \log \left ( \frac{P(z|h_t)}{P(z|c_t) }\right ) \\
        &=\mathbb{E}_{h_t \sim P(h_t)}
\left [D_{KL}(P(s_{t+1},r_t|h_t) || P(s_{t+1},r_t|\hat{\sigma}_t(h_t)) \right ]
    \end{split}
\end{align*}
\end{proof}
\end{proposition}
We now complete the proof. Combining Proposition \ref{proposition:ib_epsilon} with Proposition \ref{proposition:inf_dkl} we get that     \[\mathbb{E}_{h_t \sim P(h_t)}
\left [D_{KL}\left (P_{s_{t+1},r_t|h_t} || P_{s_{t+1},r_t|\hat{\sigma}_t(h_t)}\right ) \right ]\leq \epsilon \] as required.
\end{proof}
\end{proposition}

Let $\pi_k(\hat{\sigma}_t(h_t))$ denote the policy at iteration $k$, and note that it is defined on the information state. At iteration $k+1$, we first collect data using $\pi_k$. We denote $P_{\pi_k}(h_t)$ the probability of observing a history in this data collection process. We then use CPC to learn an approximate information state. Let $D(h_t) = D_{KL}\left (P_{s_{t+1},r_t|h_t,a_t} || P_{s_{t+1},r_t|\hat{\sigma}_t(h_t),a_t} \right )$.

\begin{proposition}\label{proposition:markov_dkl}
Let Assumption \ref{assumption:eps} hold, then \begin{equation}\label{eq:approx_info_state_bound}
    \sum_{h_t} P_{\pi_k}(h_t) D(h_t) \leq \epsilon.
\end{equation}
\end{proposition}
\begin{proof}
    Assumption \ref{assumption:eps} holds, therefore the result is an immediate corollary from \ref{proposition:dkl_eps} for every $t\in 0,1,\dots,T-1$.
\end{proof}

For some distance measure $D$, let $\Pi_\beta = \left\{ \pi : D(\pi(h_t), \pi_k(\hat{\sigma}_t(h_t))) \leq \beta \quad \forall h_t \right\}$ denote the set of policies that are $\beta$-similar to $\pi_k$.

We next define the optimal next policy $\pi^*$
\begin{equation}
    \pi^* \in \argmax_{\pi \in \Pi_\beta} \mathbb{E}^\pi \left[ \sum_{t=0}^{T-1} r(s_t,a_t) \right].
\end{equation}

Note that the value of this policy satisfies the following Bellman optimality equations:
\begin{equation}
    \begin{split}
        Q_t(h_t, a_t) &= r(h_t, a_t) + \mathbb{E}\left[ V_{t+1}(h_{t+1})\right] \\
        V_t(h_t) &= \max_{\pi: D(\pi(h_t), \pi_k(h_t)) \leq \beta} \sum_{a} \pi(a) Q_t(h_t, a),
    \end{split}
\end{equation}
for $t\leq T$, and $V_T(h_T) = 0$.

\textbf{We now present our main result, where we consider an iterative policy improvement scheme based on the approximate information state of ContraBAR and provide policy improvement bounds.}
\begingroup
\def\thetheorem{\ref{thm:approx_ais}}
\begin{theorem}
Let Assumption \ref{assumption:eps} hold for some representation $c_t$. Consider the distance function between two distributions $D(P_1(x),P_2(x)) = \max_x | P_1(x)/P_2(x) |$.
We let $\hat{r}(c_t,a_t)=\mathbb{E}[r_t|c_t,a_t]$ and $\hat{P}(c'|c_t,a_t)=\mathbb{E}[\mathbf{1}(c_{t+1}=c')|c_t,a_t]$ denote an approximate reward and transition kernel, respectively. Define the value functions
\begin{equation}
    \begin{split}
        \hat{Q}_t(c_t, a_t) &= \hat{r}(c_t, a_t) + \sum_{c_{t+1}}\hat{P}(c_{t+1}|c_t,a_t) \hat{V}_{t+1}(c_{t+1}) \\
        \hat{V}_t(c_t) &= \max_{\pi: D(\pi(c_t), \pi_k(c_t)) \leq \beta} \sum_{a} \pi(a) \hat{Q}_t(c_t, a),
    \end{split}
    \tag{\ref{eq:DP_inside_thm}}
    \raisetag{2.2em}
\end{equation}
for $t\leq T$, and $\hat{V}_T(c_T) = 0$, and
the approximate optimal policy 
\begin{equation}
\hat{\pi}(c_t)\in \argmax_{\pi: D(\pi, \pi_k(c_t)) \leq \beta} \sum_{a} \pi(a) \hat{Q}_t(c_t, a).
\tag{\ref{eq:policy_opt_inside_thm}}
\end{equation}
Let the optimal policy $\pi^*(h_t)$ be defined similarly, but with $h_t$ replacing $c_t$ in \eqref{eq:DP_inside_thm} and \eqref{eq:policy_opt_inside_thm}.
Then we have that 
\begin{equation*}
\begin{split}
    &\mathbb{E}^{\pi^*} \left[ \sum_{t=0}^{T-1} r(s_t,a_t) \right] - \mathbb{E}^{\hat{\pi}} \left[ \sum_{t=0}^{T-1} r(s_t,a_t) \right] \leq \epsilon^{1/3} R_{max} T^2 (\sqrt{2} + 4\beta^T).
\end{split}
\end{equation*}
\end{theorem}
\addtocounter{theorem}{-1}
\endgroup

\begin{proof}
Since Assumption \ref{assumption:eps} holds, Proposition \ref{proposition:markov_dkl} does as well. 

From the Markov inequality, we have $P_{\pi_k}(D(h_t) \geq n \epsilon) \leq \frac{\epsilon}{n \epsilon} = \frac{1}{n}$.

We now the define the ``Good Set'' $H_G = \{ h_t : D(h_t) < n \epsilon\}$ and the ``Bad Set'' $H_B = \{ h_t : D(h_t) \geq n \epsilon\}$.

Next, we define an auxiliary policy $\tilde{\pi}(h_t) = \begin{cases}
			\pi^*(h_t), & \text{if } h_t \in H_G\\
            \text{worst behavior}, & \text{if } h_t \in H_B
		 \end{cases}$.

We will assume that after observing $h_t \in H_B$, the policy performs as bad as possible for the rest of the episode.

Next, we bound the performance of $\tilde{\pi}$.

\begin{proposition}\label{prop:high_prob_bound_1}
We have that $\mathbb{E}^{\pi^*} \left[ \sum_{t=0}^{T-1} r(s_t,a_t) \right] - \mathbb{E}^{\tilde{\pi}} \left[ \sum_{t=0}^{T-1} r(s_t,a_t) \right] \leq 2T^2 R_{max}\beta^T/n .$
\end{proposition}
\begin{proof}
We will denote by $r_t(h_t)$ the reward at the last state-action pair. That is, for $h_t = s_0,a_0,r_0,\dots,s_{t-1},a_{t-1},r_{t-1},s_t$ we set $r_t(h_t) = r_{t-1}$. We will denote $R(h_t)$ the sum of rewards, that is, $R(h_t) = \sum_{t'=0}^{t-1} r_{t'}$.

We also denote by $P_{\pi}(h_t)$ the probability of observing history $h_t$ under policy $\pi$. 
Note that by definition $\sum_{t=0}^{T-1} \sum_{h_t} P_{\pi}(h_t) = 1$. Also, note that by the definition of the set $\Pi_\beta$, for any two policies $\pi_1,\pi_2 \in \Pi_\beta$ we have $P_{\pi_1}(h_t)/P_{\pi_2}(h_t) \leq \beta^t$.

We now claim that 
$$\mathbb{E}^{\tilde{\pi}} \left[ \sum_{t=0}^{T-1} r_t \right] \geq \mathbb{E}^{\pi^*} \left[ \sum_{t=0}^{T-1} r_t \right] - 2T^2 R_{max}\beta^T/n.$$

We first estimate the probability that policy $\tilde{\pi}$ encounters a history in $H_B$. Consider some $t\in 0,\dots,T-1$. We have that under $P_{\pi_k}$, with probability at most $1/n$, $h_t \in H_B$. Under $P_{\tilde{\pi}}$, with probability at most $\beta^t/n$, $h_t \in H_B$. From the union bound, with probability at most $T\beta^T/n$ the policy visits at least one history in $H_B$.

Let $\bar{H}_B$ denote the set of $T$-length histories that visit a history in $H_B$, and let  $\bar{H}_G$ be its complement set.

Now, note that 
\begin{equation}
\begin{split}
\mathbb{E}^{\tilde{\pi}} \left[ \sum_{t=0}^{T-1} r_t \right] &= \sum_{h_T} P_{\tilde{\pi}}(h_T) R(h_T) \\
&= \sum_{h_T\in \bar{H}_G} P_{\tilde{\pi}}(h_T) R(h_T) + \sum_{h_T\in \bar{H}_B} P_{\tilde{\pi}}(h_T) R(h_T) \\
&= \sum_{h_T\in \bar{H}_G} P_{\pi^*}(h_T) R(h_T) + \sum_{h_T\in \bar{H}_B} P_{\tilde{\pi}}(h_T) R(h_T) \\
&\geq \sum_{h_T\in \bar{H}_G} P_{\pi^*}(h_T) R(h_T) + (T\beta^T/n) T (-R_{max}) \\
&= \sum_{h_T} P_{\pi^*}(h_T) R(h_T) - \sum_{h_T\in \bar{H}_B} P_{\pi^*}(h_T) R(h_T) + (T\beta^T/n) T (-R_{max})\\
&\geq \mathbb{E}^{\pi^*} \left[ \sum_{t=0}^{T-1} r_t \right] -2T^2 R_{max}\beta^T/n 
\end{split}
\end{equation}
The third equality is from the definition of $\pi^*$. The fourth inequality relies on the reward function being bounded, i.e $R(h_T) \geq T(-R_{max})$. This alongside the fact that $\sum_{h_T \in \bar{H}_B}P_{\tilde{\pi}}(h_T) \geq (TB^T/n)$ gives us the inequality. Note that the last inequality follows from the definition of $\pi_*$, wherein the probability of visiting at least one history in $H_B$  is the same for $\pi_*$ and $\tilde{\pi}$.
\end{proof}

Next, we note that using Pinsker's inequality, we have
$d_{TV}(P_{s_{t+1},r_t|h_t,a_t},P_{s_{t+1},r_t|c_t,a_t}) \leq \sqrt{2 d_{KL}(P_{s_{t+1},r_t|h_t,a_t},P_{s_{t+1},r_t|c_t,a_t})}$, and that
\begin{equation*}
    \begin{split}
    &|\mathbb{E}[r_t|h_t, a_t] - \mathbb{E}[r_t|c_t, a_t]|
     \leq R_{max} d_{TV}(P_{s_{t+1},r_t|h_t,a_t},P_{s_{t+1},r_t|c_t,a_t}) \\
     &|\mathbb{E}[V_{t+1}|h_t, a_t] - \mathbb{E}[V_{t+1}|c_t, a_t]|
     \leq R_{max}(T-t) d_{TV}(P_{s_{t+1},r_t|h_t,a_t},P_{s_{t+1},r_t|c_t,a_t})
    \end{split}
\end{equation*}

We next prove the following result.
\begin{proposition}\label{prop:induction_bound_1}
We have that
\begin{equation}\label{eq:value_bound_1}
\begin{split}
    \hat{Q}_t(\hat{\sigma}_t(h_t), a) &\geq Q^{\tilde{\pi}}(h_t,a) -\alpha_t, \\
    \hat{V}_t(\hat{\sigma}_t(h_t)) &\geq V^{\tilde{\pi}}(h_t) -\alpha_t,
\end{split}
\end{equation}
where $\alpha_t$ satisfies the following recursion: $\alpha_{T} = 0$, and $\alpha_t = \sqrt{2 n \epsilon}R_{max} (T-t+1) + \alpha_{t+1}$.
\end{proposition}
\begin{proof}
    We prove by backward induction. The argument holds for $T$ by definition. Assume that Equation \eqref{eq:value_bound_1} holds at time $t+1$, and consider time $t$. If $h_t \in H_B$, then by definition $\hat{Q}_t(\hat{\sigma}_t(h_t), a) \geq Q^{\tilde{\pi}}(h_t,a)$, since $\tilde{\pi}$ will take the worst possible actions after observing $h_t$. Otherwise, $h_t \in H_G$ and we have
\begin{equation*}
    \begin{split}
        &Q^{\tilde{\pi}}(h_t,a) - \hat{Q}_t(\hat{\sigma}_t(h_t), a) \\
        =& \mathbb{E}[r_t|h_t, a] + \mathbb{E}\left[V^{\tilde{\pi}}_{t+1}(h_{t+1})|h_t, a\right] - \hat{r}(\hat{\sigma}_t(h_t), a) - \sum_{c_{t+1}} \hat{P}(c_{t+1}|\hat{\sigma}_t(h_t),a)\hat{V}_{t+1}(c_{t+1})\\
        =& \mathbb{E}[r_t|h_t, a] - \mathbb{E}[r_t|c_t,a_t] \\
        &+ \mathbb{E}\left[V^{\tilde{\pi}}_{t+1}(h_{t+1})|h_t, a\right] - \mathbb{E}\left[\hat{V}_{t+1}(\hat{\sigma}_{t+1}(h_{t+1}))|h_t, a\right] \\
        &+ \mathbb{E}\left[\hat{V}_{t+1}(\hat{\sigma}_{t+1}(h_{t+1}))|h_t, a\right] - \sum_{c_{t+1}} \hat{P}(c_{t+1}|\hat{\sigma}_t(h_t),a)\hat{V}_{t+1}(c_{t+1})\\
        \leq& \sqrt{2 n \epsilon}R_{max} + \alpha_{t+1} + \sqrt{2 n \epsilon}R_{max}(T-t).
    \end{split}
\end{equation*}
We note that for $h_t \in H_G$, $D(h_t) \leq n\epsilon$, yielding the $d_{TV}$ bounds.

For the second part, If $h_t \in H_B$, then by definition $\hat{V}_t(\hat{\sigma}_t(h_t)) \geq V^{\tilde{\pi}}(h_t)$. Otherwise, $h_t \in H_G$ and we have
\begin{equation*}
\begin{split}
    V^{\tilde{\pi}}_t(h_t) - \hat{V}_t(\hat{\sigma}_t(h_t)) &= \max_{\pi: D(\pi(h_t), \pi_k(h_t)) \leq \beta} \sum_{a} \pi(a) Q^{\tilde{\pi}}_t(h_t, a) - \max_{\pi: D(\pi(h_t), \pi_k(h_t)) \leq \beta} \sum_{a} \pi(a) \hat{Q}_t(\hat{\sigma}_t(h_t), a) \\
    &\leq \max_{\pi: D(\pi(h_t), \pi_k(h_t)) \leq \beta} \sum_{a} \pi(a) (\hat{Q}_t(\hat{\sigma}_t(h_t), a)+ \alpha_t) - \max_{\pi: D(\pi(h_t), \pi_k(h_t)) \leq \beta} \sum_{a} \pi(a) \hat{Q}_t(\hat{\sigma}_t(h_t), a)  \\
    &=\alpha_t.
\end{split}
\end{equation*}
\end{proof}


We next define another auxiliary policy $\tilde{\hat{\pi}}(h_t) = \begin{cases}
			\hat{\pi}(h_t), & \text{if } h_t \in H_G\\
            \text{optimal behavior}, & \text{if } h_t \in H_B
		 \end{cases}$.
We will assume that after observing $h_t \in H_B$, the policy perform optimally for the rest of the episode.
Therefore, $V^{\tilde{\hat{\pi}}}_t(h_t \in H_B) = V_t(h_t)$.
We have the following results, analogous to Propositions \ref{prop:high_prob_bound_1} and \ref{prop:induction_bound_1}.

\begin{proposition}\label{prop:high_prob_bound_2}
We have that $\mathbb{E}^{\tilde{\hat{\pi}}} \left[ \sum_{t=0}^{T-1} r(s_t,a_t) \right] - \mathbb{E}^{\hat{\pi}} \left[ \sum_{t=0}^{T-1} r(s_t,a_t) \right] \leq 2T^2 R_{max}\beta^T/n .$
\end{proposition}
\begin{proof}
    Analogous to the proof of Proposition \ref{prop:high_prob_bound_1}.
\end{proof}

\begin{proposition}\label{prop:induction_bound_2}
We have that
\begin{equation}\label{eq:value_bound_2}
\begin{split}
    \hat{Q}_t(\hat{\sigma}_t(h_t), a) &\leq Q^{\tilde{\hat{\pi}}}(h_t,a) + \alpha_t, \\
    \hat{V}_t(\hat{\sigma}_t(h_t)) &\leq V^{\tilde{\hat{\pi}}}(h_t) +\alpha_t,
\end{split}
\end{equation}
where $\alpha_t$ satisfies the following recursion: $\alpha_{T} = 0$, and $\alpha_t = \sqrt{2 n \epsilon}R_{max} (T-t+1) + \alpha_{t+1}$.
\end{proposition}
\begin{proof}
    Similarly to the proof of Proposition \ref{prop:induction_bound_1}. The argument hold for $T$ by definition. Assume that Equation \eqref{eq:value_bound_2} holds at time $t+1$, and consider time $t$. If $h_t \in H_B$, then by definition $\hat{Q}_t(\hat{\sigma}_t(h_t), a) \leq Q^{\tilde{\hat{\pi}}}(h_t,a)$, since $\tilde{\hat{\pi}}$ will take the best possible actions after observing $h_t$. Otherwise, $h_t \in H_G$ and we have
\begin{equation*}
    \begin{split}
        &\hat{Q}_t(\hat{\sigma}_t(h_t), a) - Q^{\tilde{\hat{\pi}}}(h_t,a) \\
        =& \hat{r}(\hat{\sigma}_t(h_t), a) + \sum_{c_{t+1}} \hat{P}(c_{t+1}|\hat{\sigma}_t(h_t),a)\hat{V}_{t+1}(c_{t+1}) - \mathbb{E}[r_t|h_t, a] - \mathbb{E}\left[V^{\tilde{\hat{\pi}}}_{t+1}(h_{t+1})|h_t, a\right] \\
        =& \mathbb{E}[r_t|\hat{\sigma}_t(h_t),a_t] - \mathbb{E}[r_t|h_t, a]  \\
        &+ \mathbb{E}\left[\hat{V}_{t+1}(\hat{\sigma}_{t+1}(h_{t+1}))|h_t, a\right] - \mathbb{E}\left[V^{\tilde{\hat{\pi}}}_{t+1}(h_{t+1})|h_t, a\right] \\
        &+ \sum_{c_{t+1}} \hat{P}(c_{t+1}|\hat{\sigma}_t(h_t),a)\hat{V}_{t+1}(c_{t+1}) - \mathbb{E}\left[\hat{V}_{t+1}(\hat{\sigma}_{t+1}(h_{t+1}))|h_t, a\right] \\
        \leq& \sqrt{2 n \epsilon}R_{max} + \alpha_{t+1} + \sqrt{2 n \epsilon}R_{max}(T-t).
    \end{split}
\end{equation*}
For the second part, If $h_t \in H_B$, then by definition $\hat{V}_t(\hat{\sigma}_t(h_t)) \leq V^{\tilde{\hat{\pi}}}(h_t)$. Otherwise, $h_t \in H_G$ and we have
\begin{equation*}
\begin{split}
    \hat{V}_t(\hat{\sigma}_t(h_t)) - V^{\tilde{\hat{\pi}}}_t(h_t)  &=  \max_{\pi: D(\pi(h_t), \pi_k(h_t)) \leq \beta} \sum_{a} \pi(a) \hat{Q}_t(\hat{\sigma}_t(h_t), a) - \max_{\pi: D(\pi(h_t), \pi_k(h_t)) \leq \beta} \sum_{a} \pi(a) Q^{\tilde{\hat{\pi}}}_t(h_t, a)\\
    &\leq \max_{\pi: D(\pi(h_t), \pi_k(h_t)) \leq \beta} \sum_{a} \pi(a) (Q^{\tilde{\hat{\pi}}}_t(h_t, a) + \alpha_t) - \max_{\pi: D(\pi(h_t), \pi_k(h_t)) \leq \beta} \sum_{a} \pi(a) Q^{\tilde{\hat{\pi}}}_t(h_t, a) \\
    &=\alpha_t.
\end{split}
\end{equation*}

\end{proof}

\begin{corollary}
    We have
\begin{equation*}
\begin{split}
    &\mathbb{E}^{\pi^*} \left[ \sum_{t=0}^{T-1} r(s_t,a_t) \right] - \mathbb{E}^{\hat{\pi}} \left[ \sum_{t=0}^{T-1} r(s_t,a_t) \right] 
    \leq \mathbb{E}^{\tilde{\pi}} \left[ \sum_{t=0}^{T-1} r(s_t,a_t) \right] - \mathbb{E}^{\hat{\pi}} \left[ \sum_{t=0}^{T-1} r(s_t,a_t) \right] + 2T^2 R_{max}\beta^T/n\\
    &= \sum_{h_0} P(h_0) \left( V_0^{\tilde{\pi}}(h_0) - V_0^{\hat{\pi}}(h_0)\right) + 2T^2 R_{max}\beta^T/n\\
    &\leq \sum_{h_0} P(h_0) \left( V_0^{\tilde{\pi}}(h_0) - \hat{V}_0(\hat{\sigma}_0(h_0)) + \hat{V}_0(\hat{\sigma}_0(h_0)) - V_0^{\hat{\pi}}(h_0)\right) + 2T^2 R_{max}\beta^T/n \\
    &\leq \sum_{h_0} P(h_0) \left( V_0^{\tilde{\pi}}(h_0) - \hat{V}_0(\hat{\sigma}_0(h_0)) + \hat{V}_0(\hat{\sigma}_0(h_0)) - V_0^{\tilde{\hat{\pi}}}(h_0)\right) + 4T^2 R_{max}\beta^T/n \\
    &\leq 2\alpha_0 + 4T^2 R_{max}\beta^T/n 
\end{split}
\end{equation*}
\end{corollary}
We note that the first inequality follows from Proposition \ref{prop:high_prob_bound_1}. The second equality stems from the fact that $P(h_0)=P(s_0)$, which is not affected by the choice of policy. The fourth transition follows from the addition and subtraction of $V_0^{\tilde{\hat{\pi}}}(h_0)$ and the use of Proposition \ref{prop:high_prob_bound_2}. The final inequality follows from Propositions \ref{prop:induction_bound_1} and \ref{prop:induction_bound_2}.

Let us bound $\alpha_0$. By the recursion $\alpha_t = \sqrt{2 n \epsilon}R_{max} (T-t+1) + \alpha_{t+1}$ we have that $\alpha_0 = \frac{T^2}{2}\sqrt{2 n \epsilon}R_{max}$. Setting $n = \epsilon^{-1/3}$ we obtain the desired result:
\begin{equation*}
\begin{split}
\mathbb{E}^{\pi^*} \left[ \sum_{t=0}^{T-1} r(s_t,a_t) \right] - \mathbb{E}^{\hat{\pi}} \left[ \sum_{t=0}^{T-1} r(s_t,a_t) \right]  &\leq T^2\sqrt{2} \epsilon^{1/3} R_{max} + 4T^2 R_{max}\beta^T \epsilon^{1/3} \\
&= \epsilon^{1/3} R_{max} T^2 (\sqrt{2} + 4\beta^T).
\end{split}
\end{equation*}



\end{proof}
\section{Environments}\label{appendix:environments}
\paragraph{Reacher Image:}
In this environment, a two-link planar robot needs to reach an unknown goal as in \citep{dorfman2021offline}, except that the goal is randomly chosen along a horizontal section of $0.48$. For each task, the agent receives a reward of +1 if it is within a small radius $r=0.05$ of the goal, and 0 otherwise.
\[r_t=\begin{cases}
			1, & \text{if $\lVert x_t 
 - x_{goal} \rVert_2\leq 0.05$}\\
            0, & \text{otherwise}
		 \end{cases}\]
where $x_t$ is the location of the robot’s end effector. The agent observes single-channel images of size 64 × 64 of the
environment]. The horizon is set to $50$ and we aggregate $k = 2$ consecutive episodes to form a trajectory of length $100$.
\paragraph{Panda Reacher:}
A Franka Panda robot tasked with placing the end effector at a goal on a 2d semi-circle of radius $0.15$ with fixed $z=0.15/2$ in 3d-space. The task is  adapted from the Reacher task in Panda Gym \citep{gallouedec2021pandagym}, with the goal occluded. For each task, the agent receives a reward of +1 if it is within a
small radius $r = 0.05$ of the goal, and 0 otherwise.
\[r_t=\begin{cases}
			1, & \text{if $\lVert x_t 
 - x_{goal} \rVert_2\leq 0.05$}\\
            0, & \text{otherwise}
		 \end{cases}\]
where $x_t$ is the current location of the end effector. The action space is 3-dimensional and bounded $[-1, 1]^{3}$. The agent observes a $3$-channel image of size $84 \times 84$ of the environment. We set the horizon to $50$ and aggregate $k=3$ consecutive episodes to form a trajectory of length $150$. 
\paragraph{Panda Wind:}
This environment is identical to Panda Reacher, except that the goal is fixed and for each task the agent experiences different wind with shifts the transition function, such that for an MDP $\mathcal{M}$ the transition function becomes
\[s_{t+1} = s_t + a_t + w_\mathcal{M}\]
where $w_{\mathcal{M}}$ is task specific and drawn randomly from a circle of radius $0.1$.
To get to the goal and stay there, the agent must learn to quickly adapt in a way that cancels the effect of the wind.

\section{Implementation Details}\label{appdx:impl_details}
In this section we outline our training process and implementation details, exact hyperparameters can be found in our code at \url{https://github.com/ec2604/ContraBAR}

The CPC component termed $g_{AR}$ consists of a recurrent encoder, which at time step t takes as input the tuple $(a_{t}, r_{t+1}, s_{t+1})$.
The state, reward and action are passed each through a different fc layer (or a
cnn feature-extractor for the states in image-based inputs). Our CPC projection head takes in $(c_t^a, z_{t+k})$ and passes it through one hidden layer of half the input size, with an ELU activation function.

\section{Architecture Details}\label{appdx:arch_details}

In this section we detail practical considerations regarding the CPC architecture.

In Section \ref{section:contrabar_arch} we described situations where $c_t$ does not need to encode belief regarding the task in order to distinguish between positive and negative observations. This is detrimental to learning a sound sufficient statistic as we would like $c_t$ to encode information regarding the reward and transition functions, as they are what set apart each task. In order to prevent this ``shortcut'' from being used, we can perform hard negative mining. We do this by using negative observations that cannot be distinguished from the positive observation without belief regarding the transition and reward functions. In the case where only the reward functions vary,  we can do this by taking the state and action of the positive observation and sampling a new reward function. We then calculate the respective reward and embed it as a negative observation alongside the original state and action. By having the positive and negative observations share the same state and action, we ensure that $c_t$ must be informative regarding the reward function in order to distinguish between positive and negative observations. We note that in this modified setup we use $s_t$ as the initial hidden state for the action-gru and include the original $c_t$ as input to the CPC projection head. This ensures that the gradient of the loss with respect to the action-gru does not affect $c_t$, which should encode information regarding the reward function. For the case where the environments only vary in reward functions, we propose a simpler solution which is to omit the action-gru, as the future actions except for $a_{t+k-1}$ do not affect $r_{t+k}$. We can simply use $(c_t,z_{t+k})$ as input to the CPC projection head -- we note that in this case $z_{t+k}$ is an embedding of the reward, state \textit{and action}. We found in practice that this simplification also worked well for the environments we used where the transitions varied.

The modified architecture where the action-gru is omitted can be seen in Figure \ref{fig:contrabar_nogru}. In Figure \ref{fig:antgoal_relabel} we demonstrate on the Ant-Goal environment that omitting the action-gru and reward-relabeling \textit{with} the action-gru yield similar results. Finally, we note that hard-negative mining can be done for varying transitions by sampling a random transition from the prior and simulating the transition to some $s_{t+k}$ given $s_t, a_t,\dots, a_{t+k}$.

\begin{figure}[ht]
  \centering
  \vskip 0.2in
  \begin{center}
  \centerline{\includegraphics[width=0.7\columnwidth]{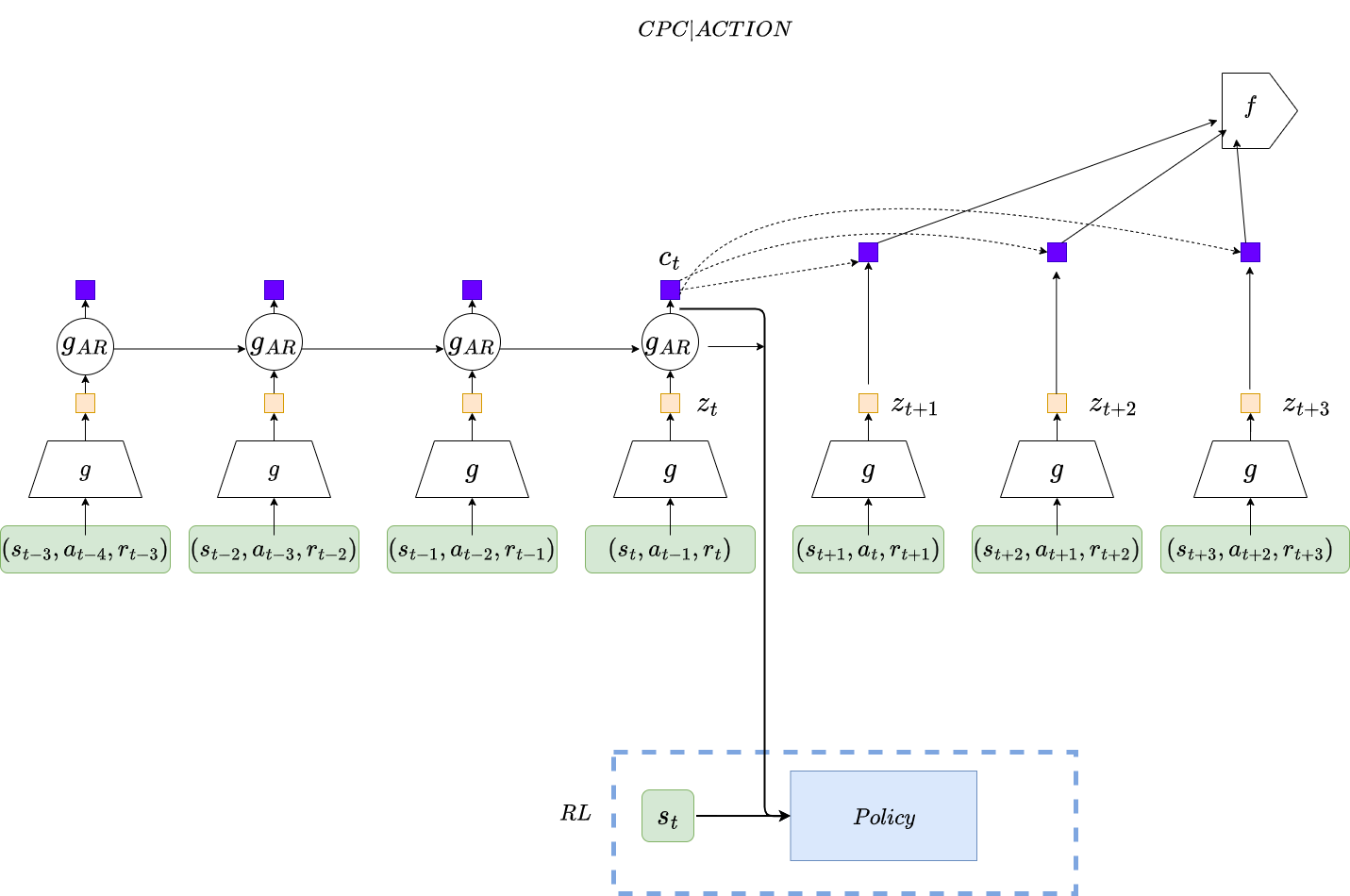}}
  \caption{ContraBAR architecture where the action-gru is omitted. The history encoding is the same, future observations however, do not include the action.}\label{fig:contrabar_nogru}
\end{center}
\vskip -0.2in
\end{figure}

\begin{figure}[!ht]
  \centering
  \begin{center}
  \centerline{\includegraphics[width=0.7\textwidth]{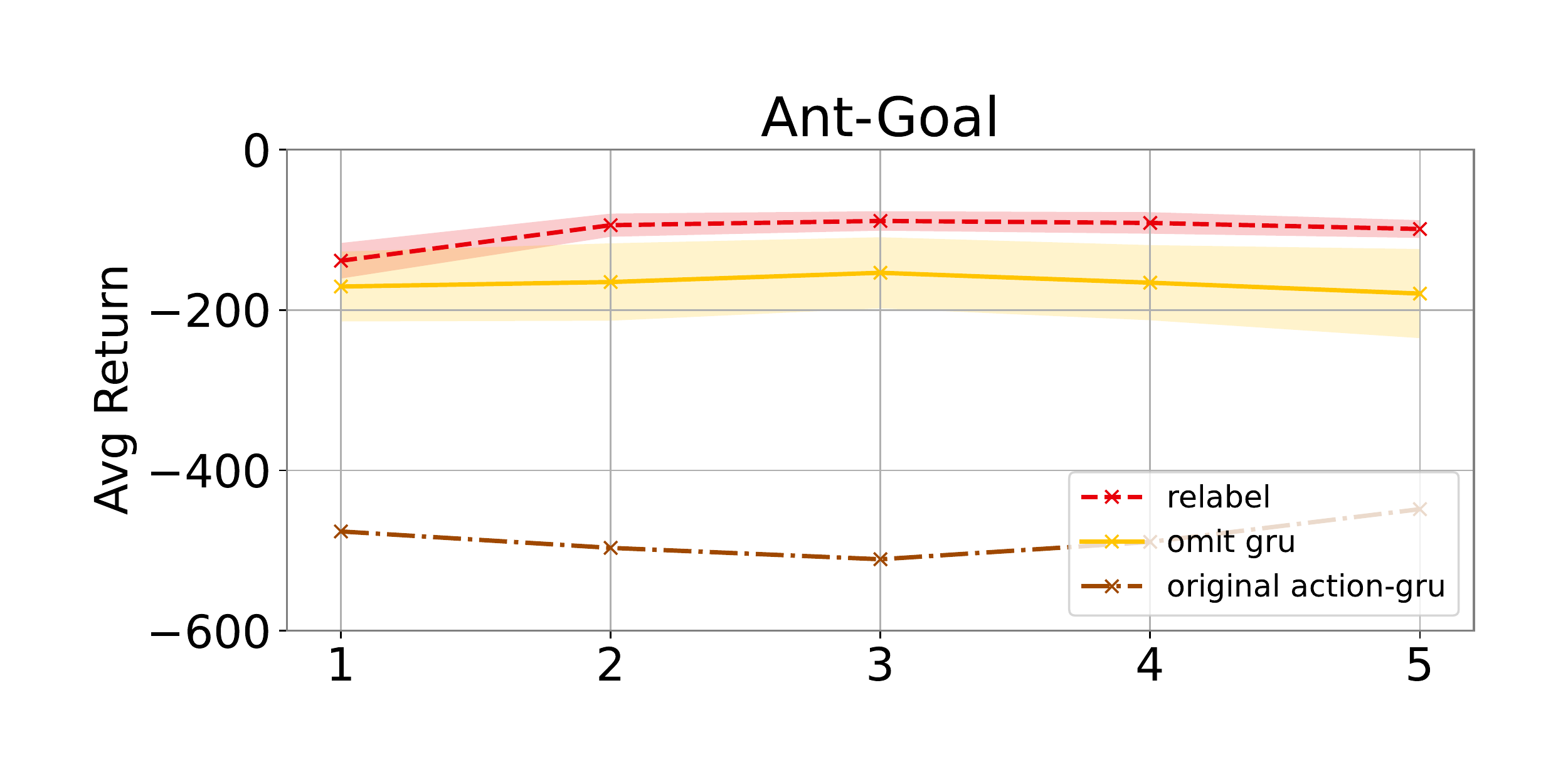}}
  \caption{The relabel and omit-gru architectures are run on the Ant-goal environment, with 10 seeds (evaluated over 10 environment samples each). The meta-trained policies are rolled out for 5 episodes to show how
they adapt to the task. Values shown are averages across tasks (95\% confidence intervals shaded). The results for both methods are similar, with slightly better results when the action-gru is omitted. We additionally show a single seed of the original action-gru architecture, evaluated on 10 environment samples -- we see that the sample mean is out of the confidence interval for both the relabel and omit-gru architectures.}\label{fig:antgoal_relabel}
\end{center}
\end{figure}

\section{Image-based inputs are computationally restrictive for VariBAD}\label{appendix:varibad_comp}
To understand the computational restriction in VariBAD \citep{zintgraf2020varibad}, we look to the formulation of the VAE objective. For every timestep $t$, the past trajectory $\tau_{:t}$ is encoded to infer the posterior $q(m|\tau_{:t})$, and used by the decoder to reconstruct \textit{the entire} trajectory including the future. In our analysis we restrict ourselves to the memory required for reconstruction of the reward trajectory, in an image-based domain, under the following assumptions: \begin{enumerate*}[label=(\arabic*)]
\item{images of dimension $d\times d \times 3$ are embedded to a representation of size $32$ via 3 convolutions with 32 channels each and kernels of size, with strides $2,2,1$ respectively.}
\item{actions are of size $2$ and embedded with a linear layer of size $16$ }
\item{The trajectory is of length 120, which is average for the domains in meta RL}
\end{enumerate*}. We draw attention to the fact that reward decoder in VariBAD receives $s_t, s_{t+1}$ as input, requiring us to take into consideration the memory required for embedding the image trajectory. On top of this, we also consider three times the size of the parameters of the image encoder (parameters, gradients and gradient moments).  We present the memory consumption as a function of the image dimensions in \cref{fig:varibad_mem}.
We note that in practice we often wish to decode multiple trajectories at once, and we also need to take into account encoder portion of the  model as well as its gradients.
\begin{figure}[!h]
   \centering
  \vskip 0.2in
  \begin{center}
  \centerline{\includegraphics[width=0.5\columnwidth]{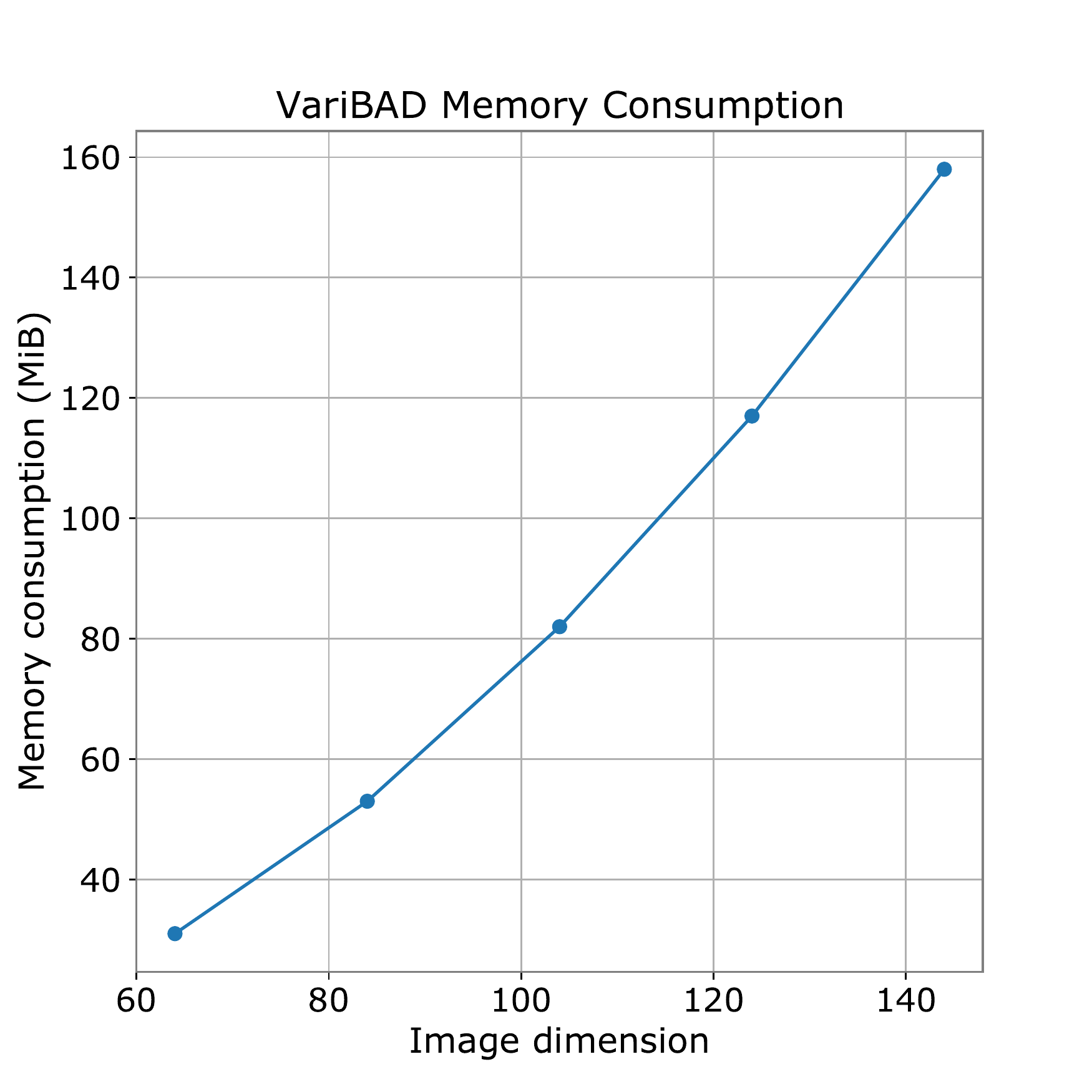}}
  \caption{VariBAD memory consumption for decoding the rewards for image-based input, as a function of the image dimension given a trajectory length of $200$}\label{fig:varibad_mem}
  \end{center}
\vskip -0.2in
\end{figure}

\end{document}